\documentclass{article}

\PassOptionsToPackage{numbers, compress}{natbib}

\usepackage[final]{neurips_2021}

\usepackage{xr-hyper}
\usepackage{xr}
\usepackage{hyperref}

\usepackage[toc,page,titletoc]{appendix}
\usepackage{minitoc}

\renewcommand\appendixname{\LARGE \centering Supplementary material for \\ Dynamic Causal Bayesian Optimisation \vspace{1cm}}
\usepackage{scrwfile}
\TOCclone[\appendixname]{toc}{atoc}
\newcommand\StartAppendixEntries{}
\AfterTOCHead[toc]{%
	\renewcommand\StartAppendixEntries{\value{tocdepth}=-10000\relax}%
}
\AfterTOCHead[atoc]{%
	\edef\maintocdepth{\the\value{tocdepth}}%
	\value{tocdepth}=-10000\relax%
	\renewcommand\StartAppendixEntries{\value{tocdepth}=\maintocdepth\relax}%
}
\newcommand*\appendixwithtoc{%
	\cleardoublepage
	\appendix
	\addtocontents{toc}{\protect\StartAppendixEntries}
	\listofatoc
}

\makeatletter
\newcommand*{\addFileDependency}[1]{% argument=file name and extension
  \typeout{(#1)}
  \@addtofilelist{#1}
  \IfFileExists{#1}{}{\typeout{No file #1.}}
}
\makeatother

\newcommand*{\myexternaldocument}[1]{%
    \externaldocument{#1}%
    \addFileDependency{#1.tex}%
    \addFileDependency{#1.aux}%
}

\myexternaldocument{supplement}
\usepackage{microtype}
\usepackage{graphicx}
\usepackage{subcaption}
\usepackage{booktabs} % for professional tables
\usepackage{cancel}
\usepackage{wrapfig}

\usepackage[titlenumbered, ruled]{algorithm2e}
\usepackage{nicefrac}       % compact symbols for 1/2, etc.
\usepackage{amssymb,amsfonts}       % blackboard math symbols
\usepackage{amsmath,stackengine,mathtools}
\usepackage{amsthm}

\newcommand{\indep}{\rotatebox[origin=c]{90}{$\models$}}
\DeclareMathOperator*{\argmax}{argmax} % thin space, limits underneath in displays
\DeclareMathOperator*{\argmin}{arg\,min}

\let\emptyset\varnothing

\theoremstyle{definition}
\newtheorem{definition}{Definition}

\theoremstyle{definition}
\newtheorem{proposition}{Proposition}[section]

\theoremstyle{remark}

\theoremstyle{definition}
\newtheorem{assumptions}{Assumptions}

\usepackage[utf8]{inputenc}

\usepackage{xspace}
\usepackage[T1]{fontenc}    % use 8-bit T1 fonts
\newtheorem{theorem}{Theorem}
\newtheorem{corollary}{Corollary}
\usepackage[capitalise]{cleveref}
\usepackage{diagbox}
\crefformat{section}{\S#2#1#3}
\crefformat{subsection}{\S#2#1#3}
\crefformat{subsubsection}{\S#2#1#3}

\usepackage{todonotes}

\usepackage{enumitem}
\setlist[itemize]{leftmargin=*}
\usepackage{multirow}

\usepackage{dsfont}
\newcommand{\mat}[1]{\mathbf{#1}}
\renewcommand{\vec}[1]{ \mathbf{#1} } % math bold
\newcommand{\dataset}{\mathcal{D}}
\newcommand{\x}{\vec{x}}
\newcommand{\X}{\vec{X}}
\newcommand{\y}{\vec{y}}

\newcommand{\W}{\mat{W}}

\newcommand{\normal}{\mathcal{N}}

\newcommand{\expectation}[2]{ \mathbb{E}_{#1}{\left[#2\right]} }

\newcommand{\pa}[1]{\text{Pa}\!\left(#1\right)}

\newcommand{\DIntXst}{\dataset^I_{s,t}}
\newcommand{\DIntXstzero}{\dataset^I_{s,t=0}}
\newcommand{\DIntXstX}{\mat{X}^I}
\newcommand{\DIntXstY}{\mat{Y}^I_{s,t}}

\newcommand{\graph}{\mathcal{G}}

\newcommand{\missets}{\mathbb{M}}

\newcommand{\DO}[2]{\operatorname{do} \!  \left(#1 = #2\right)}

\newcommand{\Ypt}{\Yt^{\acro{PT}}}
\newcommand{\Ypnt}{\Yt^{\acro{PNT}}}
\newcommand{\Apy}{\Xst^{\acro{PY}}}
\newcommand{\Bpy}{\IPrev^{\acro{PY}}}
\newcommand{\Anpy}{\Xst^{\acro{NPY}}}
\newcommand{\Bnpy}{\IPrev^{\acro{NPY}}}

\newcommand{\Apw}{\Xst^{\acro{PW}}}
\newcommand{\Bpw}{\IPrev^{\acro{PW}}}

\newcommand{\apw}{\x^{\acro{PW}}}
\newcommand{\xw}{\x^{\acro{W}}}
\newcommand{\bpw}{\iprev^{\acro{PW}}}

\newcommand{\ypt}{\yt^{\acro{PT}}}

\newcommand{\apy}{\x^{\acro{PY}}}
\newcommand{\bpy}{\iprev^{\acro{PY}}}
\newcommand{\anpy}{\x^{\acro{NPY}}}
\newcommand{\bnpy}{\iprev^{\acro{NPY}}}

\newcommand{\fW}{f_W}

\newcommand{\uW}{\mat{u}_W}

\newcommand{\R}{R}
\newcommand{\rval}{r}

\newcommand{\fyy}{f_Y^{Y}}
\newcommand{\fyny}{f_Y^{\text{NY}}}
\newcommand{\fst}{f_{s,t}}
\newcommand{\vecfopt}{\mat{f}^\star}

\newcommand{\Sst}{\mat{S}_{s,t}}
\newcommand{\Sset}{\mathbb{S}}

\newcommand{\Yt}{Y_t}
\newcommand{\Ftot}{\mat{F}_{0:t}}

\newcommand{\Ytotone}{Y_{0:t-1}}
\newcommand{\Xtotone}{\X_{0:t-1}}
\newcommand{\Vtotone}{\mat{V}_{0:t-1}}
\newcommand{\IVar}{I_{0:t-1}^{V}}
\newcommand{\ILev}{I_{0:t-1}^{L}}
\newcommand{\IPrev}{I_{0:t-1}}
\newcommand{\Xt}{\X_t}
\newcommand{\partsXt}{\mathcal{P}(\X_t)}
\newcommand{\Xst}{\X_{s,t}}

\newcommand{\xst}{\x_{s,t}}
\newcommand{\sst}{\mat{s}_{s,t}}
\newcommand{\yt}{\y_{t}}
\newcommand{\ytval}{y_{t}}
\newcommand{\yst}{y_{s,t}}
\newcommand{\epsst}{\epsilon_{s,t}}
\newcommand{\ytstar}{y^\star_t}
\newcommand{\ytotstar}{y^\star_{0:t-1}}

\newcommand{\dint}{\text{d}}

\newcommand{\EIst}{\textsc{ei}_{s,t}}

\newcommand{\iprev}{\mat{i}}
\newcommand{\w}{\mat{w}}
\newcommand{\Utot}{\mat{U}_{0:t}}
\newcommand{\mst}{m_{s,t}}
\newcommand{\kst}{k_{s,t}}
\newcommand{\sigmast}{\sigma_{s,t}}
\newcommand{\DIopt}{\dataset^{I}_{\star}}
\newcommand{\DIntXsstar}{\dataset^{I}_{s=s^\star,t}}
\newcommand{\unit}{\mu \text{mol} \cdot \text{N} \cdot \text{L}^{-1}}

\newcommand{\ie}{i.e.\xspace}
\newcommand{\eg}{e.g.\xspace}

\newcommand{\eq}{Eq.\xspace}

\newcommand{\fig}{Fig.\xspace}

\newcommand{\iid}{i.i.d.\xspace}

\newcommand{\acro}[1]{\textsc{#1}\xspace}

\newcommand{\gptext}{\acro{gp}}

\newcommand{\rbf}{\acro{rbf}}

\newcommand{\bo}{\acro{bo}}

\newcommand{\mab}{\acro{mab}}
\newcommand{\rl}{\acro{rl}}

\newcommand{\cbo}{\acro{cbo}}

\newcommand{\DAG}{\acro{dag}}

\newcommand{\mis}{\acro{mis}}

\newcommand{\ei}{\acro{ei}}
\newcommand{\sem}{\acro{scm}}

\newcommand{\EI}{\acro{ei}}

\newcommand{\dbn}{\acro{dbn}}
\newcommand{\our}{\acro{dcbo}}
\newcommand{\dgo}{\acro{dcgo}}

\newcommand{\abo}{\acro{abo}}

\newcommand{\gap}{\acro{g}}

\newcommand{\expone}{\acro{Stat.}}
\newcommand{\expnoise}{\acro{Noisy}}
\newcommand{\expmissing}{\acro{Miss.}}
\newcommand{\expcomplex}{\acro{Multiv.}}
\newcommand{\expindep}{\acro{Ind.}}
\newcommand{\expnonstat}{\acro{NonStat.}}
\newcommand{\exprealec}{\acro{Econ.}}
\newcommand{\exprealpol}{\acro{ode}}

\title{Dynamic Causal Bayesian Optimization}

\author{%
	Virginia Aglietti\thanks{Denotes equal contribution.}\\
	University of Warwick\\
	The Alan Turing Institute\\
	\texttt{V.Aglietti@warwick.ac.uk} \\
	\And 
	Neil Dhir$^*$\\
	The Alan Turing Institute\\
	\texttt{ndhir@turing.ac.uk} \\
	\And 
	Javier Gonz\'alez\\
	Microsoft Research Cambridge\\
	\texttt{Gonzalez.Javier@microsoft.com} \\
	\And
	Theodoros Damoulas \\
	University of Warwick\\
	The Alan Turing Institute\\
	\texttt{T.Damoulas@warwick.ac.uk} \\		
}

\begin{document}

\maketitle

\begin{abstract}
This paper studies the problem of performing a sequence of optimal interventions in a causal dynamical system where both the target variable of interest and the inputs evolve over time. This problem arises in a variety of domains \eg system biology and operational research. Dynamic Causal Bayesian Optimization (\our) brings together ideas from sequential decision making, causal inference and Gaussian process (\gptext) emulation. \our is useful in scenarios where all causal effects in a graph are changing over time. At every time step \our identifies a local optimal intervention by integrating both observational and past interventional data collected from the system. We give theoretical results detailing how one can transfer interventional information across time steps and define a dynamic causal \gptext model which can be used to quantify uncertainty and find optimal interventions in practice. We demonstrate how \our identifies optimal interventions faster than competing approaches in multiple settings and applications.
\end{abstract}

\section{Introduction}
\label{sec:introduction}
\begin{wrapfigure}{r}{0.35\textwidth}
\vspace{-1.25cm}
    \centering
    \includegraphics[width=0.35\textwidth]{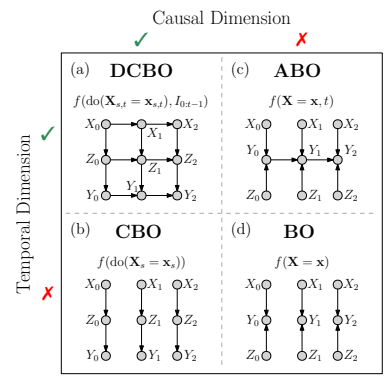}
    \caption{\DAG representation of a dynamic causal global optimisation (\dgo) problem (a) and the \DAG considered when using \cbo, \abo or \bo to address the same problem. Shaded nodes gives observed variables while the arrows represent causal effects.}
    \vspace{-4em}
    \label{fig:map_methods}
\end{wrapfigure}
% Decision making 
Solving decision making problems in a variety of domains requires understanding of cause-effect relationships in a system. This can be obtained by experimentation. However, deciding how to intervene at every point in time is particularly complex in dynamical systems, due to the evolving nature of causal effects. For instance, companies need to decide \emph{how} to allocate scarce resources across different quarters. In system biology, scientists need to identify genes to knockout at specific points in time. This paper describes a probabilistic framework that finds such optimal interventions over time.

% economic example
Focusing on a specific example, consider a setting in which $\Yt$ denotes the unemployment-rate of an economy at time $t$, $Z_t$ is the economic growth and $X_t$ the inflation rate. \fig \ref{fig:map_methods}a depicts the causal graph \citep{pearl1995causal} representing an agent's understanding of the causal links between these variables. The agent aims to determine, at each time step $t \in \{0,1,2\}$, the optimal action to perform in order to minimize the \emph{current} unemployment rate $\Yt$ while accounting for the intervention cost. 
The investigator could frame this setting as a sequence of global optimization problems and find the solutions by resorting to Causal Bayesian Optimization \citep[\cbo,][]{cbo}. \cbo extends Bayesian Optimization \citep[\bo,][]{shahriari2015taking} to cases in which the variable to optimize is part of a causal model where a sequence of interventions can be performed. However, \cbo does not account for the system's temporal evolution thus breaking the time dependency structure existing among variables (\cref{fig:map_methods}b). This will lead to sub-optimal solutions, especially in non-stationary scenarios. The same would happen when using Adaptive Bayesian Optimization \citep[\abo,][]{abo} (\cref{fig:map_methods}c) or \bo (\cref{fig:map_methods}d). \abo captures the time dependency of the objective function but neither considers the causal structure among inputs nor their temporal evolution. \bo disregards both the temporal and the causal structure. Our setting differs from both reinforcement learning (\rl) and the multi-armed bandits setting (\mab). Differently from \mab we consider interventions on continuous variables where the dynamic target variable has a non-stationary interventional distribution. In addition, compared to \rl, we do not model the state dynamics explicitly and allow the agent to perform a number of explorative interventions which do not change the underlying state of the system, before selecting the optimal action. We discuss these points further in \cref{sec:related_work}.
Dynamic Causal Bayesian Optimization\footnote{A Python implementation is available at: \url{https://github.com/neildhir/DCBO}.}, henceforth \our, accounts for both the causal relationships among input variables and the causality between inputs and outputs which might evolve over time. %This allows \our to determine how the level of traffic or the number of commercial activities should be manipulated in order minimize air pollution at every time step. 
\our integrates \cbo with dynamic Bayesian networks (\acro{dbn}), offering a novel approach for decision making under uncertainty within dynamical systems. \acro{dbn} \citep{koller2009probabilistic} are commonly used in time-series modelling and carry dependence assumptions that do not imply causation. Instead, in probabilistic causal models \citep{pearl2000causality}, which form the basis for the \cbo framework, graphs are buildt around causal information and allow us to reason about the effects of different interventions. By combining \cbo with \dbn{s}, the proposed methodology finds an optimal \emph{sequence} of interventions which accounts for the causal temporal dynamics of the system. In addition, \our takes into account past optimal interventions and transfers this information across time, thus identifying the optimal intervention faster than competing approaches and at a lower cost. We make the following contributions:
% Contributions
\begin{itemize}
\itemsep0em
\item We formulate a new class of optimization problems called Dynamic Causal Global Optimization (\dgo) where the objective functions account for the temporal causal dynamics among variables. 
\item We give theoretical results demonstrating how interventional information can be transferred across time-steps depending on the topology of the causal graph. 
\item Exploiting our theoretical results, we solve the optimization problem with \our. At every time step, \our constructs surrogate models for different intervention sets by integrating various sources of data while accounting for past interventions. 
\item We analyze \our performance in a variety of settings comparing against \cbo, \abo and \bo.
\end{itemize}
%\vspace{-0.1cm}
\subsection{Related Work}
\label{sec:related_work}
% Optimization in dynamic settings
\textbf{Dynamic Optimization} Optimization in dynamic environments has been studied in the context of evolutionary algorithms \cite{fogel1966artificial, goldberg1987nonstationary}. More recently, other optimization techniques \cite{pelta2009simple, trojanowski2009immune, de2006stochastic} have been adapted to dynamic settings, see e.g. \cite{cruz2011optimization} for a review. Focusing on \bo, the literature on dynamic settings \cite{azimi2011dynamic, bogunovic2016time, abo} is limited. 
% \citet{azimi2011dynamic} performed batch \bo in a dynamic setting where the batch sizes are dynamically determined. \citet{bogunovic2016time} introduced a \bo algorithm with bandit feedback and a reward function that varies with time. More recently, \citet{abo} developed \abo, a framework for solving \bo on continuous spaces when the function evolution follows a more complex behaviour than a simple Markov model. 
The dynamic \bo framework closest to this work is given by \citet{abo} and focuses on functions defined on continuous spaces that follow a more complex behaviour than a simple Markov model. \abo treats the inputs as fixed and not as random variables, thereby disregarding their temporal evolution and, more importantly, breaking their causal dependencies. 
% In addition, \abo requires a slow rate of change of the objective function so as to gather enough samples to learn the function evolution over space and time. All these methods tackle the dynamic dimension of the problems we address but do not account for the causal relationships among variables. 

% Causal Optimization
\textbf{Causal Optimization} Causal \bo \citep[\cbo,][]{cbo} focuses instead on the causal aspect of optimization and solves the problem of finding an optimal intervention in a \DAG by modelling the intervention functions with single \gptext{s} or a multi-task \gptext model \cite{daggp}. \cbo disregards the existence of a temporal evolution in both the inputs and the output variable, treating them as \iid overtime. While disregarding time significantly simplifies the problem, it prevents the identification of an optimal intervention at every $t$. 
%In many practical applications, the \iid assumption does not provide an adequate description for the data. Indeed, different causal methodologies %and extensions of structural causal models have been adapted to deal with longitudinal studies \cite{granger1969investigating, white2010granger, peters2012causal, hyttinen2013discovering, pfister2019invariant}. See \cite{peters2009detecting,bauer2016arrow} for a discussion of the relation between causality and the arrow of time. 

\textbf{Bandits and \rl} In the broader decision-making literature, causal relationships have been considered in the context of bandits \citep{bareinboim2015bandits, lattimore2016causal, lee2018structural, lee2019structural} and reinforcement learning \citep{lu2018deconfounding, buesing2018woulda, foerster2018counterfactual, zhang2019near, madumal2020explainable}. In these cases, actions or arms, correspond to interventions on a causal graph where there exists complex relationships between the agent’s decisions and the received rewards. While dynamic settings have been considered in acausal bandit algorithms \citep{besbes2014stochastic, villar2015multi, wu2018learning}, causal \mab have focused on static settings.  
Dynamic settings are instead considered by \rl algorithms and formalized through Markov decision processes (\acro{mdp}). In the current formulation, \our does not consider an \acro{mdp} as we do not have a notion of \emph{state} and therefore do not require an explicit model of its dynamics. The system is fully specified by the causal model. As in \bo, we focus on identifying a set of time-indexed optimal actions rather than an optimal policy. We allow the agent to perform explorative interventions that do not lead to state transitions. More importantly, differently from both \mab and \rl, \emph{we allow for the integration of both observational and interventional data}. 
An expanded discussion on the reason why \our should be used and the links between \our, \cbo, \abo and \rl is included in the supplement (\cref{sec:connections}).
\section{Background and Problem Statement}
Let random variables and values be denoted by upper-case and lower-case letters respectively. Vectors are represented shown in bold. $\DO{X}{x}$ represents an intervention on $X$ whose value is set to $x$. $p(Y \mid X=x)$ represents an observational distribution and $p(Y \mid \DO{X}{x})$ represents an interventional distribution. This is the distribution of $Y$ obtained by intervening on $X$ and fixing its value to $x$ in the data generating mechanism (see \cref{fig:dcbo_visual}), irrespective of the values of its parents. Evaluating $p(Y\mid\DO{X}{x})$ requires “real” interventions while $p(Y\mid X=x)$ only requires “observing” the system.  $\mathcal{D}^O$ and $\mathcal{D}^I$ denote observational and interventional datasets respectively.
Consider a structural causal model (\sem) defined in Definition \ref{def:scm}.
\begin{definition}{\textbf{(Structural Causal Model)} \citep[p. 203]{pearl2000causality}.}
\label{def:scm}
A structural causal model $M$ is a triple $\left\langle \mat{U},\mat{V},F) \right\rangle$ where $\mat{U}$ is a set of background variables (also called \emph{exogenous}), that are determined by factors outside of the model. $\mat{V}$ is a set $\{V_1,V_2,\ldots,V_{|\mat{V}|} \}$ of observable variables (also called \emph{endogenous}), that are determined by variables in the model (i.e., determined by variables in $\mat{U} \cup \mat{V}$). $F$ is a set of functions $\{f_1, f_2, \ldots, f_n\}$ such that each $f_i$ is a mapping from the respective domains of $U_i \cup \pa{V_i}$ to $V_i$, where $U_i \subseteq \mat{U}$ and $\pa{V_i} \subseteq \mat{V} \setminus V_i$ and the entire set $F$ forms a mapping from $\mat{U}$ to $\mat{V}$. In other words, each $ \{ f_i \in v_i \leftarrow f_i(\pa{v_i}, u_i) \mid  i = 1, \ldots, n \}$, assigns a value to $V_i$ that depends on the values of the select set of variables $(U_i \cup \pa{V_i})$.
\end{definition}
$M$ is associated to a directed acyclic graph (\DAG) $\graph$, in which each node corresponds to a variable and the directed edges point from members of $\text{Pa}(V_i)$ and $U_i$ to $V_i$. \emph{We assume $\graph$ to be known} and leave the integration with causal discovery \citep{glymour2019review} methods for future work. Within $\mat{V}$, we distinguish between three different types of variables: non-manipulative variables $\mat{C}$, which cannot be modified, treatment variables $\mat{X}$ that can be set to specific values and output variable $Y$ that represents the agent’s outcome of interest. Exploiting the rules of do-calculus \citep{pearl2000causality} one can compute $p(Y \mid \DO{X}{x})$ using observational data. This often involves evaluating intractable integrals which can be approximated by using observational data to get a Monte Carlo estimate $\widehat{p}(Y \mid \DO{X}{x}) \approx p(Y \mid \DO{X}{x})$. These approximations will be consistent when the number of samples drawn from $p(\mat{V})$ is large. 

\paragraph{Causality in time} 
One can encode the existence of causal mechanisms across time steps by explicitly representing these relationships with edges in an extended graph denoted by $\graph_{0:T}$. For instance, the \DAG in \cref{fig:map_methods}(a) can be seen as one of the \DAG{s} in \cref{fig:map_methods}(b) propagated in time. The \DAG in \cref{fig:map_methods}(a) captures both the causal structure existing across time steps and the causal mechanism within every ``time-slice'' $t$ \citep{koller2009probabilistic}. In order to reason about interventions that are implemented in a sequential manner, that is \emph{at time $t$ we decide which intervention to perform in the system} and so define:
\begin{definition}
$M_t$ is the \sem at time step $t$ defined as $M_t = \left\langle \mat{U}_{0:t}, \mat{V}_{0:t}, \Ftot\right\rangle$ where $0:t$ denotes the union of the corresponding variables or functions up to time $t$ (see \cref{fig:dcbo_visual}). $\mat{V}_{0:t}$ includes $\mat{X}_{0:t} = \mat{X}_t$, $\mat{Y}_{0:t} = Y_t$ and $\mat{C}_{0:t} = \mat{C}_{t} \cup \mat{C}_{0:t-1}$. The functions in $\Ftot$ corresponding to intervened variables are replaced by constant values while the exogenous variables related to them are excluded from $\mat{U}_{0:t}$. 
\end{definition}
\begin{definition}
$\graph_t$ is the causal graph associated to $M_t$. In $\graph_t$, the incoming edges in variables intervened at $0:t-1$ are mutilated while intervened variables are represented by deterministic nodes (squares) -- see \cref{fig:dcbo_visual}. 
\end{definition}
\begin{figure}[ht!]
    \centering
    \includegraphics[width=0.85\textwidth]{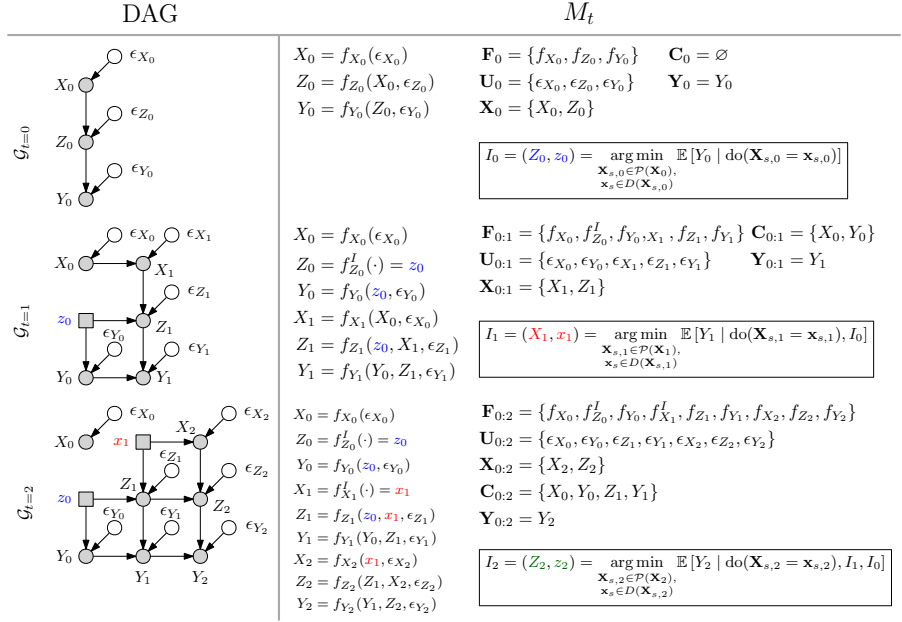}
    \caption{Structural equation models considered by \our at every time step $t \in \{0,1,2\}$. Exogenous noise variables $\epsilon_i$ are depicted here but are omitted in the remainder of the paper, to avoid clutter. For every $t$, $\graph_t$ is a mutilated version of $\graph_{t-1}$ reflecting the optimal intervention implemented in the system at $0:t-1$ which are represented by squares. The \sem functions in $\Ftot$ corresponding to the intervened variables are set to constant values. The exogenous variables that only related to the intervened variables are excluded from $U_{t}$. $\mat{C}_{0:t}$ is given by the set $\{\mat{C}_{t} \cup \mat{C}_{0:t-1} \cup \mat{Y}_{0:t-1} \cup \mat{X}_{0:t-1}\}$.
    % \virgi{make sure white/grey node is correct. Change the F and U to be $\Ftot$ and  $U_{0:t}$}
    }
    % The set of non manipulative variables at every time step is given by the union of the non manipulative variables up to time t, the previous target variables and the previous manipulative variables.}
    \label{fig:dcbo_visual}
    % \vspace{-1em}
\end{figure}

\paragraph{Dynamic Causal Global Optimization (\dgo)}
The goal of this work is to find a sequence of interventions, optimizing a target variable, \emph{at each time step}, in a causal \DAG. Given $\graph_t$ and $M_t$, at every time step $t$, we wish to optimize $\Yt$ by intervening on a subset of the manipulative variables $\Xt$. The optimal intervention variables $\mat{X}^\star_{s,t}$ and intervention levels $\xst^\star$ are given by:
\begin{equation}
    \Xst^\star,\xst^\star 
    = \argmin_{\Xst \in \partsXt, \x_s \in D(\Xst)} 
    \expectation{}{\Yt \mid \DO{\Xst}{\xst}, \mathds{1}_{t>0} \cdot \IPrev}
    \label{eq:dcgo}
\end{equation}
where $\IPrev = \bigcup_{i=0}^{t-1}\DO{\X_{s,i}^\star}{\x_{s,i}^\star}$ denotes previous interventions, $\mathds{1}_{t>0}$ is the indicator function and $\partsXt$ is the power set of $\Xt$. $D(\Xst)$ represents the interventional domain of $\Xst$. In the sequel we denote the previously intervened variables by $\IVar= \bigcup_{i=0}^{t-1}\X_{s,i}^\star$ and implemented intervention levels by 
$\ILev = \bigcup_{i=0}^{t-1}\x_{s,i}^\star$. The cost of each intervention is given by $\texttt{cost}(\Xst, \xst)$. In order to solve the problem in \cref{eq:dcgo} we make the following assumptions :
\begin{assumptions} \label{assumptions}
Denote by $\graph(t)$ the causal graph including variables at time $t$ in $\graph_{0:T}$ and let $\Ypt = \pa{Y_t}\cap Y_{0:t-1}$ be the set of variables in $\graph_{0:T}$ that are both parents of $Y_t$ and targets at previous time step. Let the set $\Ypnt = \pa{Y_t}\backslash \Ypt$ be the complement and denote by $f_{Y_t}(\cdot)$ the functional mapping for $\Yt$ in $M_t$. We make the following assumptions:
\begin{enumerate}[leftmargin=*]
    \item Invariance of causal structure: $\graph(t) = \graph(0), \forall t > 0$.
    \item Additivity of $f_{Y_t}(\cdot)$ that is $Y_t = f_{Y_t}(\pa{Y_t}) + \epsilon$ with 
    $f_{Y_t}(\pa{Y_t}) = \fyy(\Ypt) + \fyny(\Ypnt)$ where $\fyy$ and $\fyny$ are two generic unknown functions and $\epsilon \sim  \normal(0,\sigma^2)$.
    \item Absence of unobserved confounders in $\graph_{0:T}$.
\end{enumerate}
\end{assumptions}

Assumption (3) implies the absence of unobserved confounders at every time step. For instance, this is the case in \fig \ref{fig:map_methods}a. Still in this \DAG,  Assumption (2) implies $f_{\Yt}(\pa{Y_t}) = \fyy(Y_{t-1}) + \fyny(Z_t) + \epsilon_{Y_t}, \forall t>0$. Finally, Assumption (1) implies the existence of the same variables at every time step and a constant orientation of the edges among them for $t>0$. 

Notice that Assumptions \ref{assumptions} imply invariance of the causal structure \emph{within} each time-slice, \ie the structure, edges and vertices, concerning the nodes with the same time index. This means that, across time steps, both the graph and the functional relationships can change. Therefore, not only can the causal effects change significantly across time steps, but also the input dimensionality of the causal functions we model, might change. For instance, in the \DAG of \cref{fig:all_exp_dags}(c), the target function for $Y_2$ has dimensionality 3 and a function $f_{\Yt}(\cdot)$ that is completely different from the one assumed for $Y_1$ that has only two parents. We can thus model a wide variety of settings and causal effects despite this assumption. Furthermore, even though we assume an additive structure for the functional relationship on $Y$, the use of \gptext{s} allow us to have flexible models with highly non-linear causal effects across different graph structures. In the causality literature, \gptext models are well established and have shown good performances compared to parametric linear and non-linear models (see \eg \cite{silva2010gaussian, witty2020causal, zhang2012invariant}). The sum of \gptext{s} gives a flexible and computationally tractable model that can be used to capture highly non-linear causal effects while helping with interpretability \cite{duvenaud2011additive,  duvenaud2014automatic}. 

\section{Methodology}
In this section we introduce Dynamic Causal Bayesian Optimization (\our), a novel methodology addressing the problem in \cref{eq:dcgo}. We first study the correlation among objective functions for two consecutive time steps and use it to derive a recursion formula that, based on the topology of the graph, expresses the causal effects at time $t$ as a function of previously implemented interventions (see square nodes in \cref{fig:dcbo_visual}). 
% show how , at every $t$, they can be written as the sum of two terms. 
% The first term is determined by the optimal interventions implemented at previous time steps while the second term incorporates the causal dependency of the parents of $Y_t$ on the intervention set. We provide an explicit expression for this recursion for a general graph. Based on the topology of graph, we then demonstrate how the search space remains the same at every time step and a combinatorial search can be avoided. 
Exploiting these results, we develop a new surrogate model for the objective functions that can be used within a \cbo framework to find the optimal sequence of interventions. This model enables the integration of observational data and interventional data, collected at previous time-steps and interventional data collected at time $t$, thereby accelerating the identification of the current optimal intervention.
\subsection{Characterization of the time structure in a \DAG with time dependent variables}
\label{sec:recursion}
The following result provides a theoretical foundation for the dynamic causal \gptext model introduced later. In particular, it derives a recursion formula allowing us to express the objective function at time $t$ as a function of the objective functions corresponding to the optimal interventions at previous time steps. The proof is given in the appendix (\S \ref{sec:sec1_appendix}).

\begin{definition}
Consider a \DAG $\graph_{0:T}$ and the objective function $\expectation{}{\Yt \mid \DO{\Xst}{\xst}, \IPrev}$ for a generic time step $t \in \{0,\dots, T\}$. 
Denote by $\Ypt = (\pa{\Yt} \cap \Ytotone)$ the parents of $\Yt$ that are targets at previous time steps and by $\Ypnt = \pa{\Yt} \backslash \Ypt$ the remaining parents. For any $\Xst \in \partsXt$ and $\IVar \subseteq \Xtotone$ we define the following sets:
\begin{itemize}
    \item $\Apy = \Xst \cap \pa{\Yt}$ includes the variables in $\Xst$ that are parents of $\Yt$.
    \item $\Bpy = \IVar \cap \pa{\Yt}$ includes the variables in $\IVar$ that are parents of $\Yt$.
    \item $W \subset \pa{\Yt}$ such that $\pa{\Yt} = (\pa{\Yt} \cap \Ytotone) \cup \Apy \cup \Bpy \cup W$. $W$ includes variables that are parents of $\Yt$ but are not targets nor intervened variables.
\end{itemize}
\end{definition}

The values of $\IPrev$, $\Apy$, $\Bpy$ and $W$ will be denoted by $\iprev$, $\apy$, $\bpy$ and $\w$ respectively. 
\begin{theorem}{\textbf{Time operator}.} \label{theorem1}
Consider a \DAG $\graph_{0:T}$ and the related \sem satisfying Assumptions \ref{assumptions}. 
It is possible to prove that, $\forall \Xst \in \partsXt$, the intervention function $\fst(\x) = \expectation{}{Y_t \mid \DO{\Xst}{\x}, \mathds{1}_{t>0} \cdot \IPrev}$ with $\fst(\x): D(\Xst) \to \mathbb{R}$ can be written as:
\begin{align}
    \fst(\x) = \fyy(\vecfopt) + \expectation{p (\w \mid \DO{\Xst}{\x}, \iprev)}{\fyny(\apy, \bpy, \w)}
\label{eq: theorem_eq}
\end{align} 
where $\vecfopt =  \{\expectation{}{Y_{i}|\DO{\X_{s,i}^\star}{\x_{s,i}^\star}, I_{0:i-1}}\}_{Y_{i} \in \Ypt}$ that is the set of previously observed optimal targets that are parents of $\Yt$. $\fyy$ denotes the function mapping $\Ypt$ to $\Yt$ and $\fyny$ represents the function mapping $\Ypnt$ to $Y_t$.
\end{theorem}

\cref{eq: theorem_eq} reduces to $\expectation{p(\w|\DO{\Xst}{\x}, \iprev)}{\fyny(\apy, \bpy, \w)}$ when $\Yt$ does not depend on previous targets. This is the setting considered in \cbo that can be thus seen as a particular instance of \our.  Exploiting Assumptions (\ref{assumptions}), it is possible to further expand the second term in \cref{eq: theorem_eq} to get the following expression. A proof is given in the supplement (\S\ref{sec:sec1_appendix}).

\begin{corollary}
Given Assumptions \ref{assumptions}, we can write:
\begin{align}
  \expectation{p(\w|\DO{\Xst}{\x}, \iprev)}{\fyny(\apy, \bpy, \w)} = \expectation{p(\Utot)}{\fyny(\apy, \bpy, \{C(W)\}_{W \in \W})}  \label{eq:alter_expression}
\end{align}
where $p(\Utot)$ is the distribution for the exogenous variables up to time $t$ and $C(W)$ is given by:
\begin{align*}
    C(W) = 
    \begin{cases}
    \fW(\uW, \apw, \bpw) & \text{if} \quad \R = \emptyset \\
    \fW(\uW, \apw, \bpw, \rval) & \text{if} \quad \R \subseteq \Xst \cup \IVar\\
    \fW(\uW, \apw, \bpw, C(\R)) & \text{if} \quad \R \not\subseteq \Xst \cup \IVar
     \end{cases}
\end{align*}
where $\fW$ represents the functional mapping for $W$ in the \sem and $\uW$ is the set of exogenous variables with edges into $W$. $\apw$ and $\bpw$ are the values corresponding to $\Apw$ and $\Bpw$ which in turn represent the subset of variables in $\Xst$ and $\IVar$ that are parents of $W$. Finally $r$ is the value of $\R = \pa{W} \backslash (\Apy \cup \Bpw)$. 
\end{corollary}

\paragraph{Examples for \cref{eq: theorem_eq}:} For the \DAG in \cref{fig:map_methods}(a), at time $t=1$ and with $\IVar = \{Z_0\}$, we have $\expectation{}{Y \mid \DO{Z_1}{z}, I_0} = \fyy(y_0^\star) + \fyny(z)$. Indeed in this case $\W = \emptyset$, $\apy = z$ and $\vecfopt = \{y_0^\star = \expectation{}{Y_0|\DO{Z_0}{z_0}}\}$. Still at $t=1$ and with $\IVar = \{Z_0\}$, the objective function for $\mat{X}_{s,t} = \{X_1\}$ can be written as $\fyy(y_0^\star) + \expectation{p(z_1|\DO{X_1}{x}, I_0)}{\fyny(z_1)}$ as $\W = \{Z_1\}$. All derivations for these expressions and alternative graphs are given in the supplement (\S\ref{sec:sec1_appendix}).

\subsection{Restricting the search space}
\label{sec:space}
The search space for the problem in \cref{eq:dcgo} grows exponentially with $|\mat{X}_t|$ thus slowing down the identification of the optimal intervention when $\graph$ includes more than a few nodes. Indeed, a naive approach of finding $\mat{X}_{s,t}^\star$ at $t=0, \dots, T$ would be to explore the $2^{|\mat{X}_t|}$ sets in $\mathcal{P}(\mat{X}_t)$ at every $t$ and keep $2^{|\mat{X}_t|}$ models for the objective functions. In the static setting, \cbo reduces the search space by exploiting the results in \cite{lee2018structural}. In particular, it identifies a subset of variables $\missets \subseteq \mathcal{P}(\X)$ worth intervening on thus reducing the size of the exploration set to $2^{|\missets|}$. 

In our dynamic setting, the objective functions change at every time step depending on the previously implemented interventions and one would need to recompute $\missets$ at every $t$. However, it is possible to show that, given Assumptions \ref{assumptions}, the search space remains constant over time. Denote by $\missets_t$ the set $\missets$ at time $t$ and let $\missets_0$ represent the set at $t=0$ which corresponds to $\missets$ computed in \cbo. For $t>0$ it is possible to prove that: 
\begin{proposition}{\textbf{\mis in time.}} If Assumptions \ref{assumptions} are satisfied, $\mathbb{M}_t = \mathbb{M}_0$ for $t>0$.
\end{proposition}
\subsection{Dynamic Causal \gptext model}
Here we introduce the Dynamic Causal \gptext model that is used as a surrogate model for the objective functions in \cref{eq:dcgo}. The prior parameters are constructed by exploiting the recursion in \cref{eq: theorem_eq}. At each time step $t$, the agent explores the sets in $\mathbb{M}_t \subseteq \mathcal{P}(\X_t)$ by selecting the next intervention to be the one maximizing a given acquisition function. The \our algorithm is shown in \cref{alg:dcbo_alg}.

\paragraph{Prior Surrogate Model}
\label{sec:dc_model}
At each time step $t$ and for each $\Xst \in \missets_t$, we place a \gptext prior on the objective function $\fst(\x) =  \expectation{}{\Yt|\DO{\Xst}{\x}, \mathds{1}_{t>0} \cdot \IPrev}$. We construct the prior parameters exploiting the recursive expression in \cref{eq: theorem_eq}: 
\begin{align*}
 \fst(\x) & \sim \mathcal{GP}(\mst(\x), \kst(\x, \x')) \text{ where } 
 \\
 \mst(\x) &= \expectation{}{\fyy(\vecfopt) +
 \widehat{\mathbb{E}}_{}[\fyny(\apy, \bpy, \w)]} \\
 \kst(\x, \x') &= k_{\rbf}(\x, \x') + \sigmast(\x)\sigmast(\x') \text{ with }\\
 \sigmast(\x) &= \sqrt{\mathbb{V}[\fyy(\vecfopt) + \hat{\mathbb{E}}_{} \left [\fyny(\apy, \bpy, \w) \right ]}
\end{align*}
and $k_{\rbf}(\x, \x') \coloneqq  \exp(-\frac{||\x -\x'||^2}{2l^2})$ represents the radial basis function kernel \cite{rasmussen2003gaussian}. We have it that  
\begin{equation*}
\hat{\mathbb{E}}_{} \left [\fyny(\apy, \bpy, \w) \right ] = \hat{\mathbb{E}}_{p(\w|\DO{\Xst}{\x}, \iprev)} \left [\fyny(\apy, \bpy, \w) \right]
\end{equation*}
represents the expected value of $\fyny(\apy, \bpy, \w)$ with respect to $p(\w \mid \DO{\Xst}{\x}, \iprev)$ which is estimated via the do-calculus using observational data. The outer expectation in $\mst(\x)$ and the variance in $\sigmast(\x)$ are computed with respect to $p(\fyy, \fyny)$ which is also estimated using observational data. In this work we model $\fyy$, $\fyny$ and all functions in the \sem by independent \gptext{s}. 

Both $\mst(\x)$ and $\sigmast(\x)$ can be equivalently written by exploiting the equivalence in \cref{eq:alter_expression}. In both cases, this prior construction allows the integration of three different types of data: observational data, interventional data collected at time $t$ and the optimal interventional data points collected in the past. The former is used to estimate the \sem model and $p(\w \mid \DO{\Xst}{\x}, \iprev)$ via the rules of do-calculus. The optimal interventional data points at $0:t-1$ determine the shift $\fyy(\vecfopt)$ while the interventional data collected at time $t$ are used to update the prior distribution on $\fst(\x)$. Similar prior constructions were previously considered in static settings \cite{cbo, daggp} where only observational and interventional data at the current time step were used. The additional shift term appears here as there exists causal dynamics in the target variables and the objective function is affected by previous decisions. \cref{fig:toy_example} in the appendix shows a synthetic example in which accounting for the dynamic aspect in the prior formulation leads to a more accurate \gptext posterior compared to the baselines, especially when the the optimum location changes across time steps. 

\begin{wrapfigure}{R}{0.5\textwidth}
    \vspace{-1.5em}
    \begin{minipage}{.55\textwidth}
    \scalebox{0.9}{
        \begin{algorithm}[H]
        	\SetKwInOut{Input}{input}
        	\SetKwInOut{Output}{output}
        	\KwData{$\dataset^O$, $\{\DIntXstzero\}_{s \in \{0,\dots, |\missets_0|\}}$, $\graph_{0:T}$, 
        	$H$. }
        	\KwResult{Optimal intervention path $\{\Xst^\star, \xst^\star, y_t^\star\}_{t = 1}^T$}
        	\textbf{Initialise}: $\missets$, $\dataset^I_0$ and initial optimal $\DIopt = \emptyset$.\\
        	\For{$t = 0, \dots, T$}{ 
        	1. Initialise dynamic causal \gptext models for all $\Xst \in \missets_t$ using $\mathcal{D}^I_{\star,t-1}$ if $t>0$. \\
        	2. Initialise interventional dataset $\{\DIntXst\}_{s \in \{0,\dots, |\missets_t|\}}$ \\
        	\For{$h = 1, \dots, H$}{
        	1. Compute $\EIst(\x)$ for each $\Xst \in \missets_t$.\\
        	2. Obtain $(s^\star, \alpha^\star)$ \\
        	3. Intervene and augment $\DIntXsstar$ \\
        	4. Update posterior for $f_{s=s^\star, t}$
            }
            3. Return the optimal intervention $(\Xst^\star, \xst^\star)$\\
            4. Append optimal interventional data $\mathcal{D}^I_{\star,t} = \mathcal{D}^I_{\star,t-1} \cup ((\Xst^\star, \xst^\star), y^\star_t)$
        	}
        	\caption{\our}
        	\label{alg:dcbo_alg}
        \end{algorithm}
    }
    \end{minipage}
    \vspace{-3em}
\end{wrapfigure}

\paragraph{Likelihood}
Let $\DIntXst = (\DIntXstX, \DIntXstY)$ be the set of interventional data points collected for $\Xst$ with $\DIntXstX$ being a vector of intervention values and $\DIntXstY$ representing the corresponding vector of observed target values. As in standard \bo we assume each $\yst$ in $\DIntXstY$ to be a noisy observation of the function $\fst(\x)$ that is $\yst(\x) = \fst(\x) + \epsst$ with $\epsst \sim \normal(0, \sigma^2)$ for $s \in \{1, \dots, |\missets_t|\}$ and $t \in \{0,\dots, T\}$. In compact form, the joint likelihood function for $\DIntXst$ is $p(\DIntXstY \mid \fst, \sigma^2) = \normal(\fst(\DIntXstX),\sigma^2 \mat{I})$. 

\paragraph{Acquisition Function}
Given our surrogate models at time $t$, the agent selects the interventions to implement by solving a Causal Bayesian Optimization problem \cite{cbo}. The agent explores the sets in $\missets_t$ and decides where to intervene by maximizing the Causal Expected Improvement (\ei). Denote by $\ytstar$ the optimal observed target value in $\{\DIntXstY\}_{s=1}^{|\missets_t|}$ that is the optimal observed target across all intervention sets at time $t$. The Causal \EI is given by 
\begin{equation*}
\EIst(\x) = \expectation{p(\yst)}{\text{max}(\yst-\ytstar, 0)}/\texttt{cost}(\Xst, \x_{s,t}).     
\end{equation*}
Let $\alpha_1, \dots, \alpha_{|\missets_t|}$ be solutions of the optimization of $\EIst(\x)$ for each set in $\missets_t$ and $\alpha^\star := \max \{\alpha_1, \dots, \alpha_{|\missets_t|}\}$. The next best intervention to explore at time $t$ is given by $s^\star = \argmax_{s \in \{1, \cdots, |\missets_t|\}} \alpha_s.$ Therefore, the set-value pair to intervene on is $(s^\star, \alpha^\star)$. At every $t$, the agent implement $H$ \emph{explorative} interventions in the system which are selected by maximizing the Causal \ei. 
Once the budget $H$ is exhausted, the agent implements what we call the \emph{decision} intervention $I_t$, that is the optimal intervention found at the current time step, and move forward to a new optimization at $t+1$ carrying the information in $\ytotstar$. The parameter $H$ determines the level of exploration of the system and acts as a budget for the \cbo algorithm. Its value is determined by the agent and is generally problem specific.

\paragraph{Posterior Surrogate Model}
For any set $\Xst \in \missets_t$, the posterior distribution $p(\fst \mid \DIntXst)$ can be derived analytically via standard \gptext updates. $p(\fst \mid \DIntXst)$ will also be a \gptext with parameters 
\begin{align*}
\mst(\x \mid \DIntXst) &= \mst(\x) + \kst(\x, \DIntXstX)[\kst(\DIntXstX, \DIntXstX) + \sigma^2\mat{I}](\DIntXstY - \mst(\DIntXstX) \text{ and }\\ 
\kst(\x, \x' \mid \DIntXst) &= \kst(\x, \x') - \kst(\x, \DIntXstX)[\kst(\DIntXstX, \DIntXstX) + \sigma^2\mat{I}]\kst(\DIntXstX, \x').    
\end{align*}

\section{Experiments}
\label{sec:experiments}
We evaluate the performance of \our in a variety of synthetic and real world settings with \DAG{s} given in \cref{fig:all_exp_dags}. We first run the algorithm for a stationary setting where both the graph structure and the \sem do not change over time (\expone). We then consider a scenario characterised by increased observation noise (\expnoise) for the manipulative variables and a settings where observational data are missing at some time steps (\expmissing). Still assuming stationarity, we then test the algorithm in a \DAG where there are multivariate interventions in $\missets_t$ (\expcomplex). Finally, we run \our for a non-stationary graph where both the \sem and the \DAG change over time (\expnonstat). To conclude, we use \our to optimize the unemployment rate of a closed economy (\DAG in \cref{fig:dag_econ}, \exprealec) and to find the optimal intervention in a system of ordinary differential equation modelling a real predator-prey system (\DAG in \cref{fig:dag_ode}, \exprealpol). We provide a discussion on the applicability of \our to real-world problems in \cref{sec:appl_rw} of the supplement together with all implementation details. 

\textbf{Baselines} We compare against the algorithms in \cref{fig:map_methods}. Note that, by constructions, \abo and \bo intervene on all manipulative variables while \our and \cbo explore only $\missets_t$ at every $t$. In addition, both \our and \abo reduce to \cbo and \bo at the first time step. We assume the availability of an observational dataset $\dataset^O$ and set a unit intervention cost for all variables. 

\textbf{Performance metric} We run all experiments for $10$ replicates and show the average convergence path at every time step. We then compute the values of a modified ``gap'' metric\footnote{This metric is a modified version of the one used in \cite{huang2006global}.} across time steps and with standard errors across replicates. The metric is defined as
\begin{equation}
\gap_t = \left[\frac{y(\xst^\star) - y(\x_{\text{init}})}{y^{\star} - y(\x_{\text{init}})} + \frac{H - H(\xst^\star)}{H} \right] \Big/ \left(1 + \frac{H - 1}{H} \right)
\end{equation}
where $y(\cdot)$ represents the evaluation of the objective function, $y^\star$ is the global minimum, and $\x_{\text{init}}$ and $\xst^\star$ are the first and best evaluated point, respectively. The term $\frac{H - H(\xst^\star)}{H}$ with $H(\xst^\star)$ denoting the number of explorative trials needed to reach $\xst^\star$ captures the speed of the optimization. This term is equal to zero when the algorithm is not converged and equal to $(H-1)/H$ when the algorithm converges at the first trial. We have $0 \leq \gap_t \leq 1$ with higher values denoting better performances. For each method we also show the average percentage of replicates where the optimal intervention set $\Xst^\star$ is identified.

\begin{figure}[ht!]
    \centering
    \begin{subfigure}[t]{0.2\textwidth}
        \centering
        \includegraphics[width=\textwidth]{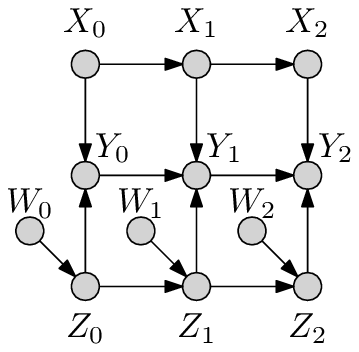}
        \caption{\expcomplex \label{fig:dag_instrument}}
    \end{subfigure}%
    \hspace*{\fill}
    \begin{subfigure}[t]{0.2\textwidth}
        \centering
        \includegraphics[width=0.84\textwidth]{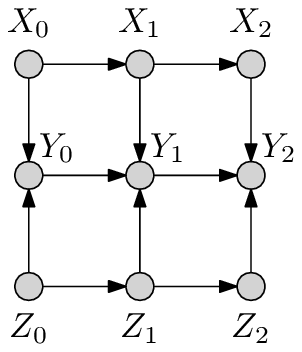}
        \caption{\expindep \label{fig:dag_independent}}
    \end{subfigure}%
    \hspace*{\fill}
    \begin{subfigure}[t]{0.2\textwidth}
        \centering
        \includegraphics[width=0.875\textwidth]{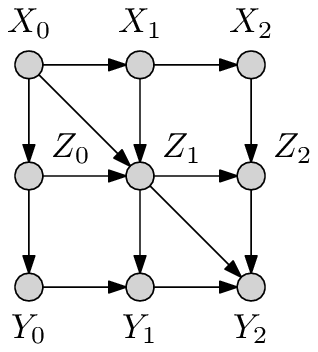}
        \caption{\expnonstat \label{fig:dag_nonstat}}
    \end{subfigure}%    
    % \\
    \begin{subfigure}[t]{0.2\textwidth}
        \centering
        \includegraphics[width=1.07\textwidth]{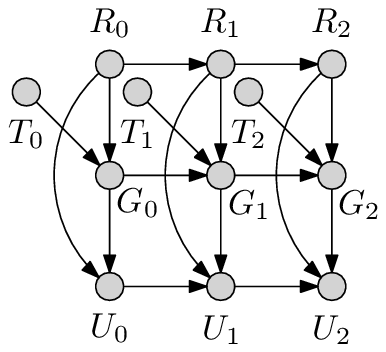}
        \caption{\exprealec \label{fig:dag_econ}}
    \end{subfigure}%
    \begin{subfigure}[t]{0.2\textwidth}
        \centering
        \includegraphics[width=0.75\textwidth]{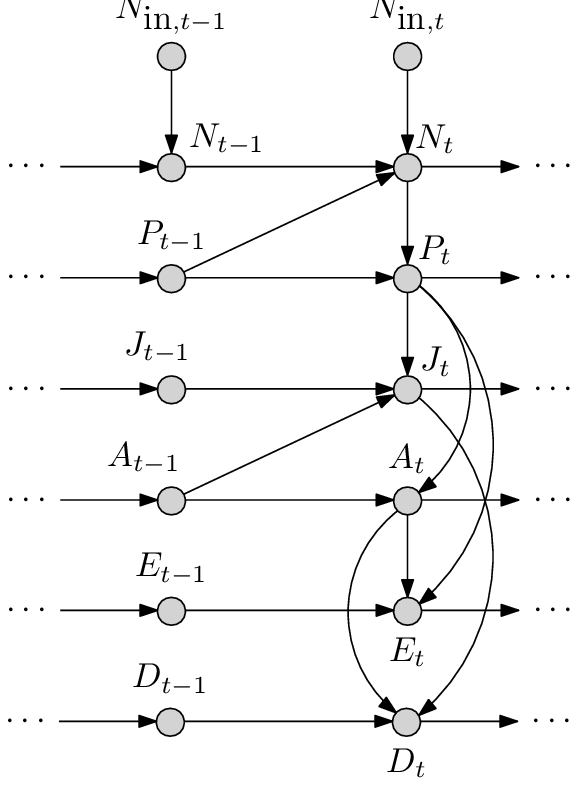}
        \caption{\exprealpol. \label{fig:dag_ode}}
    \end{subfigure}%
    \caption{\DAG{s} used in the experimental sections for the real (\cref{sec:real_experiments}) and synthetic data (\cref{sec:synthetic}).}
    \label{fig:all_exp_dags}
    % \vspace{-1em}
\end{figure}

\subsection{Synthetic Experiments}
\label{sec:synthetic}

\textbf{Stationary \DAG and \sem (\expone)} 
We run the algorithms for the \DAG in \cref{fig:map_methods}(a) with $T = 3$ and $N=10$. \emph{For $t>0$, \our converges to the optimal value faster than competing approaches} (see \cref{fig:toy_example} in the supplement, right panel, 3\textsuperscript{rd} row). \our identifies the optimal intervention set in $93\%$ of the replicates (\cref{tab:table_sets}) and reaches the highest average gap metric (\cref{tab:table_gaps}). In this experiment the location of the optimum changes significantly both in terms of optimal set and intervention value when going from $t=0$ to $t=1$. This information is incorporated by \our through the prior dependency on $\ytotstar$. In addition, \abo performance improves over time as it accumulates interventional data and uses them to fit the temporal dimension of the surrogate model. This benefits \abo in a stationary setting but might penalise it in non-stationary setting where the objective functions change significantly. 

\textbf{Noisy manipulative variables (\expnoise):} \emph{The benefit of using \our becomes more apparent when the manipulative variables observations are noisy} while the evolution of the target variable is more accurately detected.
In this case both the convergence of \our and \cbo are slowed down by noisy observations which are diluting the information provided by the do-calculus making the priors less informative. However, the \our prior dependency on $\ytotstar$ allows it to correctly identify the shift in the target variable thus improving the prior accuracy and the speed-up of the algorithm (\cref{fig:exp_noise}).

\begin{figure*}
\centering
\includegraphics[width=\linewidth]{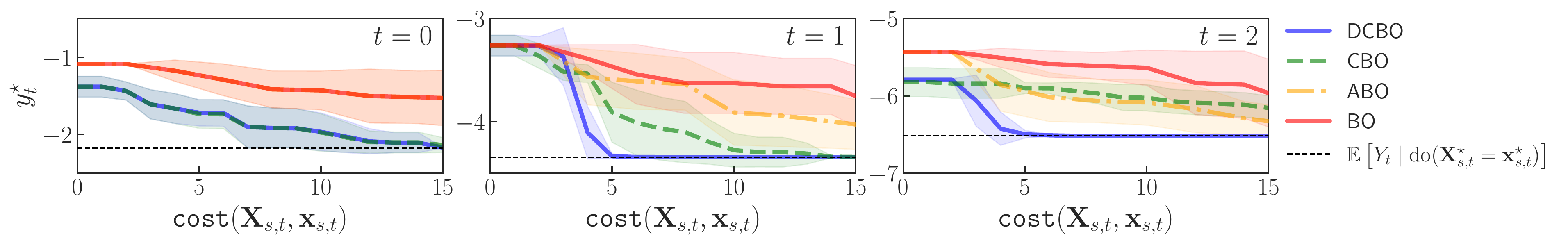}
\caption{ Experiment \expnoise. Convergence of \our and competing methods across replicates. The dashed black line (- - -) gives the optimal outcome $y^*_t, \forall t$. Shaded areas are $\pm$ one standard deviation.}
\label{fig:exp_noise}
\end{figure*}

\textbf{Missing observational data (\expmissing)}
\emph{Incorporating dynamic information in the surrogate model allows us to efficiently optimise a target variable even in setting where observational data are missing}. We consider the \DAG in \cref{fig:map_methods}(a) with $T=6$, $N=10$ for the first three time steps and $N=0$ afterwards. \our uses the observational distributions learned with data from the first three time steps to construct the prior for $t>3$. On the contrary, \cbo uses the standard prior for $t>3$. In this setting \our consistently outperforms \cbo at every time step. However, \abo performance improves over time and outperforms \our starting from $t=4$ due to its ability to exploit all interventional data collected over time (see \cref{fig:missing_exp} in the supplement). 

\textbf{Multivariate intervention sets (\expcomplex)} 
\emph{When the optimal intervention set is multivariate, both \our and \cbo convergence speed worsen}. For instance, for the \DAG in \cref{fig:dag_instrument}, $|\missets| = 5$ thus both \cbo and \our will have to perform more explorative interventions before finding the optimum. At the same time, \abo and \bo consider interventions only on $\{W_t, X_t, Z_t\}, \forall t$ and need to explore an even higher intervention space. The performance of all methods decrease in this case (\cref{tab:table_gaps}) but \our still identifies the optimal intervention set in $93\%$ of the replicates (\cref{tab:table_sets}).

\textbf{Independent manipulative variables (\expindep):} \emph{Having to explore multiple intervention sets significantly penalises \our and \cbo when there is no causal relationship among manipulative variables} which are also the only parents of the target. This is the case for the \DAG in \cref{fig:dag_independent} where the optimal intervention is $\{X_t, Z_t\}$ at every time step. In this case, exploring $\missets$ and propagating uncertainty in the causal prior slows down \our convergence and decreases both its performance (\cref{tab:table_gaps}) and capability to identify the optimal intervention set (\cref{tab:table_sets}). 

\textbf{Non-stationary \DAG and \sem (\expnonstat):} \emph{\our outperforms all approaches in non-stationary settings where both the \DAG and the \sem change overtime} -- see \cref{fig:dag_nonstat}. Indeed, \our can timely incorporate changes in the system via the dynamic causal prior construction while \cbo, \bo and \abo need to perform several interventions before accurately learning the new objective functions.
\begin{table}[hb!]
% \vspace{-1em}
\centering
\caption{Average $\gap_t$ across 10 replicates and time steps. See \cref{fig:map_methods} for a summary of the baselines. Higher values are better. The best result for each experiment in bold. Standard errors in brackets. }
\label{tab:table_gaps}
\resizebox{\columnwidth}{!}{
%@{\extracolsep{4pt}}@{\kern\tabcolsep}
\begin{tabular}{lcccccccc}
\toprule
&  \multicolumn{6}{c}{Synthetic data} & \multicolumn{2}{c}{Real data} \\
\cmidrule(lr){2-7} 
\cmidrule(lr){8-9}
& \expone & \expmissing & \expnoise &  \expcomplex & \expindep & \expnonstat & \exprealec & \exprealpol \\
\cmidrule(lr){2-7} 
\cmidrule(lr){8-9}
\multirow{2}{*}{$\our$} & \textbf{0.88}   &  \textbf{0.84} & \textbf{0.75} & \textbf{0.49} & 0.48 & \textbf{0.69}  & \textbf{0.64} & \textbf{0.67}  \\
          & (0.00) & (0.01)  & (0.00) & (0.01) & (0.04) & (0.00) & (0.01) & (0.00)\\
\multirow{2}{*}{$\cbo$} & 0.70  & 0.70 &  0.51  & 0.48 & 0.47  & 0.61 & 0.61 & 0.65\\
          & (0.01) & (0.02) & (0.02) & (0.09) & (0.07) & (0.00) & (0.01) & (0.00)\\
\multirow{2}{*}{$\abo$} & 0.56  & 0.49 & 0.49 & 0.39 & \textbf{0.54} & 0.38 & 0.57  & 0.48\\
          & (0.01) & (0.02) & (0.04) & (0.21) & (0.01) & (0.02) & (0.02) & (0.01)\\
\multirow{2}{*}{$\bo$} & 0.54  & 0.48 & 0.38 & 0.35 & 0.50 & 0.38 & 0.50 & 0.44\\
          & (0.02) & (0.03) & (0.05) & (0.08) & (0.01) & (0.03) & (0.01) & (0.03)\\
\bottomrule
\end{tabular}
}
% \vspace{-1em}
\end{table}

\begin{table}
% \vspace{-1.5em}
\centering
\caption{Average \% of replicates across time steps for which $\Xst^\star$ is identified. See \cref{fig:map_methods} for a summary of the baselines. Higher values are better. The best result for each experiment in bold. }
\label{tab:table_sets}
\resizebox{\columnwidth}{!}{
\begin{tabular}{lcccccccc}
\toprule
&  \multicolumn{6}{c}{Synthetic data} & \multicolumn{2}{c}{Real data} \\
\cmidrule(lr){2-7} 
\cmidrule(lr){8-9}
& \expone & \expmissing & \expnoise &  \expcomplex & \expindep & \expnonstat & \exprealec & \exprealpol \\
\cmidrule(lr){2-7} 
\cmidrule(lr){8-9}
\multirow{1}{*}{$\our$} & \textbf{93.00} & 58.00 & \textbf{100.00} & \textbf{93.00} & 93.00 & \textbf{100.00} & 86.67 & \textbf{33.3}\\
\multirow{1}{*}{$\cbo$} & 90.00 & \textbf{85.00} & 90.00 & 90.0 & 90.00 & \textbf{100.00} & \textbf{93.33} & \textbf{33.3}\\
\multirow{1}{*}{$\abo$} & 0.00 & 0.00 & 0.00 & 0.00 & \textbf{100.00} & 0.00 & 66.67 & 0.00 \\
\multirow{1}{*}{$\bo$} & 0.00 & 0.00 & 0.00 & 0.00 & \textbf{100.00} & 0.00 & 66.67 & 0.00 \\
\bottomrule
\end{tabular}
}
% \vspace{-1em}
\end{table}

\subsection{Real experiments}
\label{sec:real_experiments}
\textbf{Real-World Economic data (\exprealec)} We use \our to minimize the unemployment rate $U_t$ of a closed economy. We consider its causal relationships with economic growth ($G_t$), inflation rate ($R_t$) and fiscal policy ($T_t$)\footnote{The causality between economic variables is oversimplified in this example thus the results cannot be used to guide public policy and are only meant to showcase how \our can be used within a real application.}. Inspired by the economic example in \cite{huang2019causal} we consider the \DAG in \cref{fig:dag_econ} where $R_t$ and $T_t$ are considered manipulative variables we need to intervene on to minimize $\log(U_t)$ at every time step. Time series data for 10 countries\footnote{Data were downloaded from \texttt{https://www.data.oecd.org/ [Accessed: 01/04/2021]}. All details in the supplement.} are used to construct a non-parametric simulator and to compute the causal prior for both \our and \cbo. \our convergences to the optimal intervention faster than competing approaches (see \cref{tab:table_gaps} and \cref{fig:economy_example} in the appendix). The optimal sequence of interventions found in this experiment is equal to $\{(T_0, R_0) = (9.38, -2.00), (T_1, R_1) = (0.53, 6.00), (T_2) = (0.012)\}$ which is consistent with domain knowledge. 

\textbf{Planktonic predator–prey community in a chemostat (\exprealpol)} We investigate a biological system in which two species interact, one as a predator and the other as prey, with the goal of identifying the intervention reducing the concentration of dead animals in the chemostat – see $D_t$ in \cref{fig:dag_ode}. We use the system of ordinary differential equations (\acro{ode}) given by \citep{blasius2020long} as our \sem and construct the \DAG by rolling out the temporal variable dependencies in the \acro{ode} while removing graph cycles. Observational data are provided in \citep{blasius2020long} and are use to compute the dynamic causal prior. \our outperforms competing methods in term of average gap metric and identifies the optimum faster (\cref{tab:table_gaps}). Additional details can be found in the supplement (\cref{sec:ode_details}).
\section{Conclusions}
 We consider the problem of finding a sequence of optimal interventions in a causal graph where causal temporal dependencies exist between variables. We propose the Dynamic Causal Bayesian Optimization (\our) algorithm which finds the optimal intervention at every time step by intervening in the system according to a causal acquisition function. Importantly, for each possible intervention we propose to use a surrogate model that incorporates information from previous interventions implemented in the system. This is constructed by exploiting theoretical results establishing the correlation structure among objective functions for two consecutive time steps as a function of the topology of the causal graph. We discuss the \our performance in a variety of setting characterized by different \DAG properties and stationarity assumptions. Future work will focus on extending our theoretical results to more general \DAG structures thus allowing for unobserved confounders and a changing \DAG topology within each time step. In addition, we will work on combining the proposed framework with a causal discovery algorithm so as to account for uncertainty in the graph structure.
 
 \section*{Acknowledgements}
This work was supported by the EPSRC grant EP/L016710/1, The Alan Turing Institute under EPSRC grant EP/N510129/1, the Defence and Security Programme at The Alan Turing Institute, funded by the UK Government and the Lloyds Register Foundation programme on Data Centric Engineering through the London Air Quality project. TD acknowledges support from UKRI Turing AI Fellowship (EP/V02678X/1).

\newpage
\bibliographystyle{icml2021}
\bibliography{references}

\newpage
\appendixwithtoc

\newpage

\section{Nomenclature}
\label{sec:nomenclature}
\begin{table}[htbp]
% \caption{Summary of notation used in the paper.}
\begin{center}% used the environment to augment the vertical space
% between the caption and the table
% \resizebox{\columnwidth}{!}{
% \setlength{\tabcolsep}{10pt} % Default value: 6pt
\renewcommand{\arraystretch}{1.3} % Default value: 1
\begin{tabular}{ccl}
% \toprule
\textbf{Symbol} &  & \textbf{Description} \\
% \toprule
 & & \\
$\mat{V}_{t}$ & & Set of observable variables at time $t$ \\
$\mat{V}_{0:T}$ & & Union of observable variables at time $t = 0,\dots, T$  \\
$\mat{X}_t$ & & Manipulative variables at time $t$ \\
$Y_t$ & & Target variable at time $t$ \\
$\mathcal{P}(\mat{X}_t)$ &  & Power set of $\mat{X}_t$ \\
$\mathbb{M}_t$ &  & Set of \mis sets at time $t$   \\
$\mat{X}_{s,t}$ &  & $s$-th intervention set at time $t$ \\
$\dataset$ &  & Observational dataset $\{\mat{V}^i_{0:T}\}_{i=1}^N$\\
$N$ &  &  Number of observational data points collected from the system \\
$\DIntXst$ &  & Interventional data points  collected for the intervention set $\Xst$\\
$\DIntXstX$ & & Vector of interventional values \\
$\DIntXstY$ & & Vector of target values obtained by intervening on $\Xst$ at $\DIntXstX$ \\
$H$ &  &  Maximum number of explorative interventions an agent can conduct at every $t$ \\
$I_{0:t-1}$ & & Decision Interventions at time step $0$ to $t-1$\\
$\fst$ &  & Objective function for the set $\Xst$
\\
$\mst$ &  & Prior mean function of \gptext on $\fst$
\\
$\kst$ &  & Prior kernel function of \gptext on $\fst$ \\
$\mst(\cdot\mid \DIntXst)$ &  & Posterior mean function for \gptext on $\fst$ \\
$\kst(\cdot, \cdot\mid\DIntXst)$ &  & Posterior covariance function for \gptext on $\fst$ 
\\
% \bottomrule
\end{tabular}
% }
\end{center}
\label{tab:TableOfNotation}
\end{table}

\newpage

\section{Characterization of the time structure in a \DAG with time dependent variables}
\label{sec:sec1_appendix}

In this section we give the proof for Theorem 1 in the main text. Consider the objective function $\expectation{}{\Yt|\DO{\Xst}{\xst}, \IPrev}$ and define the following sets:
\vspace{-\topsep}

\begin{itemize}
    \setlength\itemsep{1em}
    \item $\pa{\Yt} = \Ypt \cup \Ypnt$ with $\Ypt = \pa{\Yt} \cap \Ytotone$ denoting the parents of $\Yt$ that are target variables at previous time steps and $\Ypnt = \pa{\Yt} \backslash \Ypt$ including the parents of $\Yt$ that are not target variables.
    % \vspace{-0.7cm}
    \item For any set $\Xst \in \partsXt$, $\Apy = \Xst \cap \pa{\Yt}$ includes the variables in $\Xst$ that are parents of $\Yt$ while $\Anpy = \Xst  \backslash \Apy$ so that $\Xst = \Apy \cup \Anpy$.
    % \vspace{-0.7cm}
    \item For any set $\IVar \subseteq \Xtotone$, $\Bpy = \IVar \cap \pa{\Yt}$ includes the variables in $\IVar$ that are parents of $\Yt$ and $\Bnpy = \IVar  \backslash \Bpy$ so that $\IVar  = \Bpy \cup \Bnpy$.
    % \vspace{-0.7cm}
    \item For any two sets $\Xst \in \pa{\Yt}$ and $\IVar \subseteq \Xtotone$, $\W$ is a set such that $\pa{\Yt} = \Ypt \cup \Apy \cup \Bpy \cup \W$. This means that $\W$ includes those variables that are parents of $\Yt$ but are nor target at previous time steps nor intervened variables.
\end{itemize}

In the following proof the values of $\IVar$, $\Apy$, $\Bpy$ and $W$ are denoted by $\iprev$, $\apy$, $\bpy$ and $\w$ respectively. The values of $\Ypt$, $\Anpy$ and $\Bnpy$ are instead represented by $\ypt$, $\anpy$ and $\bnpy$. Finally, $\fyy$ and $\fyny$ are the functions in the \sem for $\Yt$ (see Assumptions (1) in the main text). 

\paragraph{\textit{Proof of Theorem 1}}
%Consider the objective functions $\expectation{}{\Yt|\DO{\Xst}{\xst}, \Iprev}$. 
Under Assumptions 1 we can write :
\begin{align}
    \expectation{}{\Yt|\DO{\Xst}{\xst}, \IPrev} 
     &= 
    \int \ytval p(\ytval|\DO{\Xst}{\xst}, \IPrev)\dint \ytval \nonumber
    \\& = 
    \int \cdots \int \ytval p(\ytval|\DO{\Apy}{\apy},\DO{\Anpy}{\anpy},  \Bpy, \Bnpy, \ypt, \w) \nonumber
    \\& 
    \;\;\;\;\;\;\;\;\;\;\;\;\;\;\;\; 
    \times p(\ypt, \w|\DO{\Xst}{\xst}, \IPrev)\dint \ytval \dint \ypt \dint \w \nonumber
    % \\ \text{Rule 2 and Rule 1 of do-calculus} \longrightarrow \nonumber
    \\&= \Big/   \text{Rule 2 and Rule 1 of do-calculus}    \Big/ \nonumber
    \\&=
    \int \cdots \int \ytval p(\ytval|\DO{\Apy}{\apy}, \Bpy, \ypt, \w) \nonumber
    \\&
    \;\;\;\;\;\;\;\;\;\;\;\;\;\;\;\; 
    \times p(\ypt, \w|\DO{\Xst}{\xst}, \IPrev) \dint \ytval\dint\ypt \dint \w \label{eq:inter_1}
    \\&= 
    \int \cdots \int \expectation{}{\Yt|\DO{\Apy}{\apy}, \Bpy, \ypt, \w} \nonumber
    \\&
    \;\;\;\;\;\;\;\;\;\;\;\;\;\;\;\; 
    \times p(\ypt, \w|\DO{\Xst}{\xst}, \IPrev)\dint \ypt \dint \w \nonumber
    % \\ \text{Assumption (2)} \longrightarrow \nonumber
    \\&= \Big/   \text{Assumption (2)}    \Big/ \nonumber
    \\&= \int \cdots \int \fyy(\ypt) + \fyny(\apy, \bpy, \w) \nonumber
    \\& 
    \;\;\;\;\;\;\;\;\;\;\;\;\;\;\;\; 
    \times p(\ypt, \w|\DO{\Xst}{\xst}, \IPrev)\dint\ypt \dint \w \label{eq:inter_2}
    \\&= \int \cdots \int \fyy(\ypt) p(\ypt, \w|\DO{\Xst}{\xst}, \IPrev)\dint \ypt \dint \w \nonumber
    \\ &+ \int \cdots \int \fyny(\apy, \bpy, \w) p(\ypt, \w|\DO{\Xst}{\xst}, \IPrev)\dint \ypt \dint \w \nonumber
    \end{align}
\onecolumn

\begin{align}
    &= \int  \fyy(\ypt) p(\ypt|\DO{\Xst}{\xst}, \IPrev)\dint\ypt 
    \\& + \int \fyny(\apy, \bpy, \w) p(\w|\DO{\Xst}{\xst}, \IPrev)\dint \w  \nonumber
    % \text{Time assumption} \longrightarrow  \nonumber
    \\&= \Big/   \text{Time assumption}    \Big/ \nonumber
    \\&= \int  \fyy(\ypt) p(\ypt|\IPrev)\dint \ypt + \int \fyny(\apy, \bpy, \w) p(\w|\DO{\Xst}{\xst}, \IPrev) \dint \w \label{eq:inter_3}
    % \text{Observed interventions} \longrightarrow & \nonumber
    \\&= \Big/   \text{Observed interventions}    \Big/ \nonumber
    \\&= \fyy(\vecfopt) + \int \fyny(\apy, \bpy, \w) p(\w|\DO{\Xst}{\xst}, \IPrev)\dint \w \label{eq:inter_4}
    \\&= \fyy(\vecfopt) + \expectation{p(\w|\DO{\Xst}{\xst}, \IPrev)}{\fyny(\apy, \bpy, \w)} \label{eq:ob_1}
\end{align}
with $\vecfopt =  \{\expectation{}{Y_{i}|\DO{\X_{s,i}^\star}{\x_{s,i}^\star}, I_{0:i-1}}\}_{Y_{i} \in \Ypt}$ denoting the values of $\Ypt$ corresponding to the optimal interventions implemented at previous time steps .
\eq \eqref{eq:inter_1} follows from $\Yt \indep  (\Anpy \cup \Bnpy)|\Apy, \Bpy, \W, \Ypt \;\;   \text{in} \;\; \graph_{\overline{\Apy, \Bpy}\underline{\Anpy, \Bnpy}}$ (Rule 2 of do-calculus) and $\Yt \indep (\Anpy \cup \Bnpy)|\Apy, \Bpy, \W, \Ypt \;\;   \text{in} \;\; \graph_{\overline{\Apy, \Bpy}}$ (Rule 1 of do-calculus). \eq \eqref{eq:inter_2} follows from the second assumption in Assumptions (1) in the main text. \eq \eqref{eq:inter_3} follows from $\Ypt \indep \Xst$ as interventions at time $t$ cannot affect variables at time steps $0:t-1$. Finally, noticing that $p(\ypt|\IPrev)$ is the distribution targeted when optimizing the objective function at previous time steps one can obtain \eq \eqref{eq:ob_1}.

\hfill $\blacksquare$

The derivations above show how the objective function at time $t$ is given by the expected value of the output of the functional relationship $\fyny$ where the expectation is taken with respect to the variables that are not intervened on. This expectation is then shifted to account for the interventions implemented in the system at previous time steps that are affecting the target variable through $\fyy$. Notice that, given our assumption on the absence of unobserved confounders, the distribution $p(\w|\DO{\Xst}{\xst}, \IPrev)$ can be further simplified by conditioning on the variables in $\graph$ that are on the back-door path between $(\Xst, \IPrev)$ and $\Yt$ and are not colliders. When the variable $\Yt$ does not depend on the previous target nodes, the function $\fyy$ does not exist and \eq \eqref{eq:ob_1} reduces to 
\begin{equation}
\expectation{p(\w|\DO{\Xst}{\xst}, \IPrev)}{\fyny(\apy, \bpy, \w)}. 
\end{equation}
In this case previous interventions impact the target variable at time $t$ by changing the distributions of the parents of $\Yt$ that are not intervened but the information in $\vecfopt$ is lost. 

\eq \eqref{eq:ob_1} can be further manipulated to reduce the second term to a do-free expression. Instead of applying the rules of do-calculus, one can expand $p(\w|\DO{\Xst}{\xst}, \IPrev)$ by further conditioning on the parents of $\W$ that are not in $(\Xst \cup \IPrev)$. In this case, $\w$ in $\fyny(\apy, \bpy, \w)$ is replaced by the functions $\{f_W(\cdot)\}_{W \in \W}$ in the \sem corresponding to the variables in $\W$ and computed in $\w$. This leads to a partial composition of $\fyny$ with $\{f_W(\cdot)\}_{W \in \W}$ and can be repeated recursively until the set of variables with respect to which we are taking the expectation is a subset of $\Xst$ or $\IVar$ thus making the distribution a delta function. For instance, when $\W \subset \Xst$ in \eq \eqref{eq:ob_1}, we have $p(\w|\DO{\Xst}{\xst}, \IPrev) = \delta(\w = \xw))$ where $\xw$ are the values in $\xst$ corresponding to the variables in $\W$. Therefore, \eq \eqref{eq:ob_1} reduces to $\fyy(\vecfopt) + {\fyny(\apy, \bpy, \xw)}$. 

For a generic $W \in \W \not\subseteq (\Xst \cup \IVar)$, denote by $\Apw$ and $\Bpw$ the subset of variables in $\Xst$ and $\IPrev$ that are parents of $W$ with corresponding values $\apw$ and $\bpw$. Let $R = \pa{W} \backslash (\Apw \cup \Bpw)$ and $r$ be the corresponding value. We can define the $C(\cdot)$ function as: 
\begin{align}
    C(W) =  \begin{cases}
    \fW(\uW, \apw, \bpw) &\quad\text{if} \quad R = \emptyset \\
       \fW(\uW, \apw, \bpw, r) &\quad\text{if} \quad R \subseteq \Xst \cup \IVar\\
\fW(\uW, \apw, \bpw, C(R)) &\quad\text{if} \quad R \not\subseteq \X_{s, t} \cup \IVar
     \end{cases}
     \label{eq:cfunction}
\end{align}
with $u_W$ representing the exogenous variables with edges into $W$ and $\fW$ denoting the functional mapping for $W$ in the \sem. Note that if $R = \emptyset$ and $\Apw$ and $\Bpw$ are also empty then $\fW(\uW, \apw, \bpw)$ reduces to $\fW(\uW)$. The same holds for the other cases that is $\fW(\uW, \apw, \bpw, r) = \fW(\uW, r)$ and $\fW(\uW, \apw, \bpw, C(R)) = \fW(\uW, C(R))$ when $\Apw, \Bpw = \emptyset$. Exploiting \eq \eqref{eq:cfunction} we can rewrite \eq \eqref{eq:ob_1} as:

\begin{align}
    \expectation{}{\Yt|\DO{\Xst}{\xst}, \IPrev} 
    &= \fyy(\vecfopt) + \expectation{p(\Utot)}{\fyny(\apy, \bpy, \{C(W)\}_{W \in \W})}
    \label{eq:ob_2}
\end{align}

The distribution $p(\Utot)$ can be further simplified to consider only the exogenous variables with outgoing edges into the variables on the directed paths between $\Xst$ and $\Ypnt$ and between $\IVar$ and $\Ypnt$. Notice how the second term in \eq \eqref{eq:ob_2} propagates the interventions, both at the present and past time steps, through the \sem so as to express the parents of the target variable as a function of the intervened values. The expected target is then obtained as the propagation of the intervened variables and intervened targets through the function $f_{Y_t}$ in the \sem. 

\newpage

\section{Example derivations}
\label{sec:example_derivations}

% \begin{figure}[ht!]
% \centering
% \includegraphics[width =0.8\textwidth]{figures/fig_1_appendix.eps}
% \caption{(a) \DAG 1 (b) \DAG 2}
% \label{fig:dag_derivations}
% \end{figure}

Next we show how one can use Theorem 1 to derive some of the objective functions used by \our for the \DAG{s} in \fig \ref{fig:dag_derivations}.

\begin{figure*}[ht!]
    \centering
    \begin{subfigure}[t]{0.45\textwidth}
        \centering
        \includegraphics[width=0.8\textwidth]{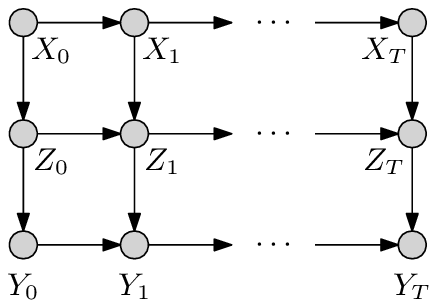}
        \caption{\DAG 1 \label{fig:dag_derivations_1}}
    \end{subfigure}%
    \hspace*{\fill}
    \begin{subfigure}[t]{0.45\textwidth}
        \centering
        \includegraphics[width=0.8\textwidth]{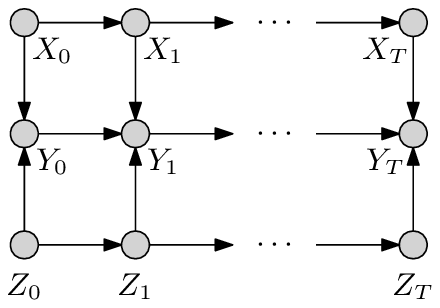}
        \caption{\DAG 2 \label{fig:dag_derivations_2}}
    \end{subfigure}%
    \caption{Dynamic Bayesian networks with different topologies. \Cref{fig:dag_derivations_1} shows a \DAG in which (per time-slice) the manipulative variable $X$ flows through $Z$, whereas in \cref{fig:dag_derivations_2} the manipulative variables are independent of each other (note the direction of the vertical edges).}
    \label{fig:dag_derivations}
\end{figure*}

\subsection{Derivations for \DAG 1 in \cref{fig:dag_derivations_1}}

Consider the \DAG in \cref{fig:dag_derivations_1} and assume that the optimal intervention implemented at time $t=0$ is given by $I_0 = \DO{Z_0}{z_0^\star}$ and gives a target value of $y_0^\star$. At $t=1$ the target variable is $Y_1$, $\Ypt = \{Y_0\}$ and $\Ypnt = \{Z_1\}$. Given $I_0$ we have $\Bpy = \emptyset$ and $\Bnpy = Z_0$. We can write the objective functions by noticing that, for $\mat{X}_{s,1} = \{Z_1\}$ we have $\Apy = \{Z_1\}$, $\Anpy = \emptyset$ and $W=\emptyset$, while for $\mat{X}_{s,1} = \{X_1\}$ we have $\Apy = \emptyset$, $\Anpy = \{X_1\}$ and $W=\{Z_1\}$. We do not compute the objective function for $\mat{X}_{s,1} = \{X_1, Z_1\}$ as this is equal to the function for $\mat{X}_{s,1} = \{Z_1\}$. Starting with $\mat{X}_{s,1} = \{Z_1\}$ we have:

\begin{align*}
    \expectation{}{Y_1|\DO{Z_1}{z}, I_0} &= \int y_1 p(y_1|\DO{Z_1}{z}, I_0) \text{d} y_1
    \\& = \int \int y_1 p(y_1|y_0, \DO{Z_1}{z}, I_0) p(y_0|\DO{Z_1}{z}, I_0)\text{d} y_1 \text{d} y_0
    \\&= \int \expectation{}{Y_1|y_0, \DO{Z_1}{z}} p(y_0|\DO{Z_1}{z}, I_0)\text{d} y_0
    \\&= \int [\fyy(y_0) + \fyny(z)] p(y_0|I_0)\text{d} y_0
    \\&= \int\fyy(y_0)p(y_0|I_0)\text{d} y_0 + \fyny(z)
    \\&= \fyy(y_0^\star) + \fyny(z)
\end{align*}
Notice that here $\Apy = \{Z_1\}$, $\Bpy = \emptyset$ and $W = \emptyset$ thus $\expectation{p(\w|\DO{\Xst}{\xst}, \IPrev)}{\fyny(\apy, \bpy, \w)} = \fyny(z)$. The objective function for $\mat{X}_{s,1} = \{X_1\}$ can be written as:

\begin{align}
    \expectation{}{Y_1|\DO{X_1}{x}, I_0} &= \int y_1 p(y_1|\DO{X_1}{x}, I_0) \text{d} y_1 \nonumber
    \\& = \int \int \int y_1 p(y_1|y_0, z_1, \DO{X_1}{x}, I_0) p(y_0, z_1|\DO{X_1}{x}, I_0)\text{d} y_1 \text{d} y_0 \text{d} z_1 \nonumber
    \\& = \int \int \int y_1 p(y_1|y_0, z_1) p(y_0, z_1|\DO{X_1}{x}, I_0)\text{d} y_1 \text{d} y_0 \text{d} z_1 \nonumber
    \\& = \int \int \expectation{}{Y_1|y_0, z_1} p(y_0, z_1|\DO{X_1}{x}, I_0) \text{d} y_0 \text{d} z_1 \nonumber
    \\& = \int \int [\fyy(y_0) + \fyny(z_1)] p(y_0, z_1|\DO{X_1}{x}, I_0) \text{d} y_0 \text{d} z_1 \nonumber
    \\& = \int \int \fyy(y_0)p(y_0, z_1|\DO{X_1}{x}, I_0) \text{d} y_0 \text{d} z_1  \\ & +\int \int \fyny(z_1) p(y_0, z_1|\DO{X_1}{x}, I_0) \text{d} y_0 \text{d} z_1 \nonumber
    \\& = \int \fyy(y_0)p(y_0|I_0) \text{d} y_0 +\int \int \fyny(z_1) p(z_1|\DO{X_1}{x}, I_0)\text{d} z_1 \nonumber
    \\& = \fyy(y_0^\star) + \int \fyny(z_1) p(z_1|\DO{X_1}{x}, I_0)\text{d} z_1 \label{eq:last_eq_X}
\end{align}

% \begin{align}
%     \\& = \int \int \fyy(y_0)p(y_0, z_1|\DO{X_1}{x}, I_0) \text{d} y_0 \text{d} z_1  +\int \int \fyny(z_1) p(y_0, z_1|\DO{X_1}{x}, I_0) \text{d} y_0 \text{d} z_1 \nonumber
%     \\& = \int \fyy(y_0)p(y_0|I_0) \text{d} y_0 +\int \int \fyny(z_1) p(z_1|\DO{X_1}{x}, I_0, y_0)p(y_0|\DO{X_1}{x}, I_0) \text{d} y_0 \text{d} z_1 \nonumber
%     \\& = \int \fyy(y_0)p(y_0|I_0) \text{d} y_0 +\int \fyny(z_1) p(z_1|\DO{X_1}{x}, I_0, y_0^\star) \text{d} z_1 \nonumber
%     \\& = \fyy(y_0^\star) + \int \fyny(z_1) p(z_1|\DO{X_1}{x}, I_0)\text{d} z_1 \label{eq:last_eq_X}
% \end{align}

In this case $\Apy = \emptyset$, $\Bpy = \emptyset$ and $\W =\{Z_1\}$ thus 
\begin{equation}
\expectation{p(\w|\DO{\Xst}{\xst}, \IPrev)}{\fyny(\apy, \bpy, \w)} = \expectation{p(z_1|\text{do}(X_1 = x), I_0)}{\fyny(z_1)}.     
\end{equation}
We can further expand \eq \eqref{eq:last_eq_X} noticing that in this case $\W = \{Z_1\} \not\subseteq \{X_1, Z_0\}$ but $\X_{s, t}^{PW} = \{X_1\}$, $I_{0:t-1}^{PW} = \{Z_0\}$ and $R = \emptyset$. Therefore we have $C(Z_1) = f_{Z_1}(\epsilon_{Z_1}, x_1, z_1)$ and \eq \eqref{eq:last_eq_X} becomes:

\begin{align*}
    \expectation{}{Y|\DO{X_1}{x}, I_0}  &= \fyy(y_0^\star) + \int \fyny(z_1) p(z_1|\DO{X_1}{x}, I_0)\text{d} z_1 
    \\&= \fyy(y_0^\star) + \int \int \fyny(z_1) p(z_1|\epsilon_{Z_1}, \DO{X_1}{x}, I_0)p(\epsilon_{Z_1}|\DO{X_1}{x}, I_0)\text{d} z_1 \text{d} \epsilon_{Z_1} 
    \\&= \fyy(y_0^\star) + \int \int \fyny(z_1) \delta(z_1 = f_{Z_1}(\epsilon_{Z_1}, x, z_0^\star))p(\epsilon_{Z_1})\text{d} z_1 \text{d} \epsilon_{Z_1} 
    \\& = \fyy(y_0^\star) + \expectation{p(\epsilon_{Z_1})}{\fyny(f_{Z_1}(\epsilon_{Z_1}, x, z_0^\star))}.
\end{align*}

\subsection{Derivations for \DAG 2 in \cref{fig:dag_derivations_2}}

Next we consider the \DAG in \cref{fig:dag_derivations_2} and assume that the optimal interventions implemented at time $t=0$ and $t=1$ are given by $I_0 = \DO{X_0}{x_0^\star}$ and $I_1 = \DO{Z_1}{z_1^\star}$. The optimal target values associated with these two interventions are given by $y_0^\star$ and $y_1^\star$ respectively. We are interested in computing two objective functions: $\expectation{}{Y_2|\DO{X_2}{x_2}, I_0, I_1}$ and $\expectation{}{Y_2|\DO{Z_2}{z_2}, I_0, I_1}$. In this case $\ypt = \{Y_1\}$, $\Ypnt =\{X_2, Z_2\}$, $\Bpy = \emptyset$ and $\Bnpy = \{X_0, Z_1\}$. Starting from $\expectation{}{Y_2|\DO{X_2}{x_2}, I_0, I_1}$, when $\mat{X}_{s,2} = \{X_2\}$ we have $\Apy = \{X_2\}$, $\Anpy = \emptyset$ and $W = \{Z_2\}$. We can write:

\begin{align}
    \expectation{}{Y_2|\DO{X_2}{x_2}, I_0, I_1} &= \int y_2 p(y_2|\DO{X_2}{x_2}, I_0, I_1) \text{d}y_2 \nonumber
    \\&= \int\int\int y_2 p(y_2|y_1, z_2, \DO{X_2}{x_2}, I_0, I_1) p(y_1, z_2| \DO{X_2}{x_2}, I_0, I_1)\text{d}y_2 \text{d}y_1 \text{d}z_2 \nonumber
    \\& = \int\int\int y_2 p(y_2|y_1, z_2, \DO{X_2}{x_2}) p(y_1, z_2| \DO{X_2}{x_2}, I_0, I_1)\text{d}y_2 \text{d}y_1 \text{d}z_2 \nonumber
    \\& = \int\int \expectation{}{Y_2|y_1, z_2, \DO{X_2}{x_2}} p(y_1, z_2| \DO{X_2}{x_2}, I_0, I_1) \text{d}y_1 \text{d}z_2 \nonumber
    \\& = \int\int [\fyy(y_1) + \fyny(x_2, z_2)] p(y_1, z_2| \DO{X_2}{x_2}, I_0, I_1) \text{d}y_1 \text{d}z_2 \nonumber
     \\& = \int\int \fyy(y_1)p(y_1, z_2| \DO{X_2}{x_2}, I_0, I_1) \text{d}y_1 \text{d}z_2  \nonumber \\& + \int \int \fyny(x_2, z_2) p(y_1, z_2| \DO{X_2}{x_2}, I_0, I_1) \text{d}y_1 \text{d}z_2 \nonumber
     \\& = \int \fyy(y_1)p(y_1| I_0, I_1) \text{d}y_1 + \int \fyny(x_2, z_2) p(z_2| \DO{X_2}{x_2}, I_0, I_1) \text{d}z_2 \nonumber
     \\& = \fyy(y_1^\star) + \int \fyny(x_2, z_2) p(z_2| I_1) \text{d}z_2 \nonumber
     \\& = \fyy(y_1^\star) + \expectation{p(\epsilon_{Z_2})}{\fyny(x_2, f_{Z_2}(z_1^\star, \epsilon_{Z_2}))} \nonumber
\end{align}

Next we compute $\expectation{}{Y_2|\DO{Z_2}{z_2}, I_0, I_1}$ by noticing that, when $\mat{X}_{s,2} = \{Z_2\}$, we have $\Apy = \{Z_2\}$, $\Anpy = \emptyset$ and $W = \{X_2\}$. In this case we have:

\begin{align}
    \expectation{}{Y_2|\DO{Z_2}{z_2}, I_0, I_1} &= \int y_2 p(y_2|\DO{Z_2}{z_2}, I_0, I_1) \text{d}y_2 \nonumber
    \\&= \int\int\int y_2 p(y_2|y_1, x_2, \DO{Z_2}{z_2}, I_0, I_1) p(y_1, x_2| \DO{Z_2}{z_2}, I_0, I_1)\text{d}y_2 \text{d}y_1 \text{d}x_2 \nonumber
    \\& = \int\int\int y_2 p(y_2|y_1, x_2, \DO{Z_2}{z_2}) p(y_1, x_2| \DO{Z_2}{z_2}, I_0, I_1)\text{d}y_2 \text{d}y_1 \text{d}x_2 \nonumber
    \\& = \int\int \expectation{}{Y_2|y_1, x_2, \DO{Z_2}{z_2}} p(y_1, x_2| \DO{Z_2}{z_2}, I_0, I_1) \text{d}y_1 \text{d}x_2 \nonumber
    \\& = \int\int [\fyy(y_1) + \fyny(x_2, z_2)] p(y_1, x_2| \DO{Z_2}{z_2}, I_0, I_1) \text{d}y_1 \text{d}x_2 \nonumber
    \\& = \int  \fyy(y_1) p(y_1| I_0, I_1) \text{d}y_1 + \int \fyny(x_2, z_2) p(x_2| \DO{Z_2}{z_2}, I_0, I_1) \text{d}x_2 \nonumber
    \\& = \fyy(y_1^\star) + \int \fyny(x_2, z_2) p(x_2| \DO{Z_2}{z_2}, I_0, I_1) \text{d}x_2 \label{eq:last_eq_y_2_z_2}
\end{align}

Let's now focus on \cref{eq:last_eq_y_2_z_2}. Here $\W=\{X_2\} \not\subseteq \{Z_2, X_0, Z_1\}$, $\Apw  = \emptyset$, $\Bpw = \emptyset$ and $R = \{X_1\}$. Therefore we have $C(X_2) = f_{X_2}(\epsilon_{X_2}, C(R))$ as $R \not\subseteq \{Z_2, X_0, Z_1\}$. We thus need to compute $C(R) = C(X_1)$. When $W=X_1$, $\X_{s, t}^{PW} = \emptyset$ but $I_{0:t-1}^{PW} = \{X_0\}$ and $R = \emptyset$. We can thus write $C(X_2) = f_{X_2}(\epsilon_{X_2}, f_{X_1}(\epsilon_{X_1}, x_0))$ and replace it in \eq \eqref{eq:last_eq_y_2_z_2} to get:

\begin{align}
    \expectation{}{Y_2|\DO{Z_2}{z_2}, I_0, I_1} &= \fyy(y_1^\star) + \expectation{p(\epsilon_{X_2})p(\epsilon_{X_1})}{\fyny(f_{X_2}(\epsilon_{X_2}, f_{X_1}(\epsilon_{X_1}, x_0)), z_2)}. \nonumber
\end{align}

\newpage

\section{Reducing the search space}
In this section we give the proof for Proposition 3.1 in the main text. Denote by $\missets_t \subseteq \partsXt$ the set of \mis{s} at time $t$ and let $\Sset_t = \partsXt \backslash \missets_t$ include the sets that are not \mis. For any set $\Xst \in \Sset_t$ we denote the \emph{superfluous} variables by $\Sst$. These are the variables not needed in the computation of the objective functions that is those variables for which $\expectation{}{\Yt|\DO{\Xst}{\xst}, \IPrev} = \expectation{}{\Yt|\DO{\Xst'}{\xst'}, \IPrev}$ where $\Xst' = \Xt \backslash \Sst$. Given the initial set of \mis{s} at time $t=0$ represented by $\missets_0$ we have:

\begin{proposition}{\textbf{Minimal intervention sets in time.}}
If $\graph_t = \graph, \forall t $ then $\missets_t = \missets_0$ for $t>0$.
\end{proposition}

\begin{proof}
Consider a generic set $\Xst \in \Sset_t$. The corresponding objective function can be written as:

\begin{align}
    \expectation{}{\Yt|\DO{\Xst}{\xst}, \IPrev} & = \expectation{}{\Yt|\DO{\Xst'}{\xst'}, \DO{\Sst}{\sst}, \IPrev} \nonumber 
    \\& 
    = \int \expectation{}{\Yt|\DO{\Xst'}{\xst'}, \DO{\Sst}{\sst}, \IPrev, \Vtotone \backslash \IPrev} \nonumber 
    \\& 
    \;\;\;\;\; \times p(\Vtotone \backslash \IPrev|\DO{\Xst'}{\xst'}, \DO{\Sst}{\sst}, \IPrev) \dint \Vtotone \nonumber 
    \\& 
    = \int \expectation{}{\Yt|\DO{\Xst'}{\xst'}, \IPrev, \Vtotone \backslash \IPrev} \label{eq:remove_S} 
    \\& 
    \;\;\;\;\; \times p(\Vtotone \backslash  \IPrev|\DO{\Xst'}{\xst'} ,\IPrev) \dint \Vtotone \nonumber 
    \\&
    = \expectation{}{\Yt|\DO{\Xst'}{\xst'}, \IPrev}
\end{align}

where \eq \eqref{eq:remove_S} can be obtained by noticing that $\Yt \indep \Sst | \Xst', \IPrev, \Vtotone \backslash \IPrev$ in $\graph_{\overline{\Sst, \IPrev, \Xst'}}$. This is due to the fact that $\Sst$ does not have back door paths to $\Yt$ in $\graph_{\overline{\Sst, \IPrev, \Xst'}}$ and its front door paths to $\Yt$ in $\graph_{\overline{\Sst, \IPrev, \Xst'}}$ are blocked by $\Xst'$. Indeed, $\Sst$ cannot have outgoing edges to variables in $0:t-1$ and the front door paths to $\Yt$ going through variables at time $t$ are blocked by definition of a \mis set by $\Xst'$ in $\graph_t = \graph, \forall t$.
\end{proof}

\newpage

\section{Additional experimental details and results}
\label{sec:additional_exp}
This section contains additional experimental details associated to the experiments discussed in Section \ref{sec:experiments} of the main text.

\subsection{Stationary \DAG and \sem (\expone)} 

The \sem used for this experiments is given by:
\begin{align*}
    X_t &= X_{t-1}\mathds{1}_{t>0} + \epsilon_X\\
    Z_t &=  \exp(-X_t) + Z_{t-1}\mathds{1}_{t>0} + \epsilon_Z\\
    Y_t &= \text{cos}(Z_t) - \exp(-Z_t/20) + Y_{t-1}\mathds{1}_{t>0} + \epsilon_Y
\end{align*}
where $\epsilon_i \sim \mathcal{N}(0,1)$ for $i \in \{X,Z,Y\}$ and  $\mathds{1}_{t>0}$ represent an indicator function that is equal to one $t>0$ and zero otherwise. We run this experiment 10 times by setting $T=3$, $N=10$, $D(X_t) = \{-5.0, 5.0\}$ and $D(Z_t) = \{-5.0, 20.0\}$. Notice that given the \DAG (\fig X) we have $\missets_t = \{\{X_t\}, \{Z_t\}\}$.

The right panel of \fig \ref{fig:toy_example} shows the true objective functions together with the optimal intervention per time step ($1^{\text{st}}$ row), the dynamic causal \gptext model for the intervention on $Z$ ($2^{\text{nd}}$ row) and the convergence of the \our algorithm to the optimum ($3^{\text{rd}}$ row). Notice how the location of the optimum changes significantly both in terms of optimal set and intervention value when going from $t=0$ to $t=1$.  \our quickly identifies the optimum via the prior dependency on $\ytotstar$. 
\begin{figure}[ht!]
\centering
\begin{subfigure}{.3\textwidth}
  \centering
  \includegraphics[width=1.\linewidth]{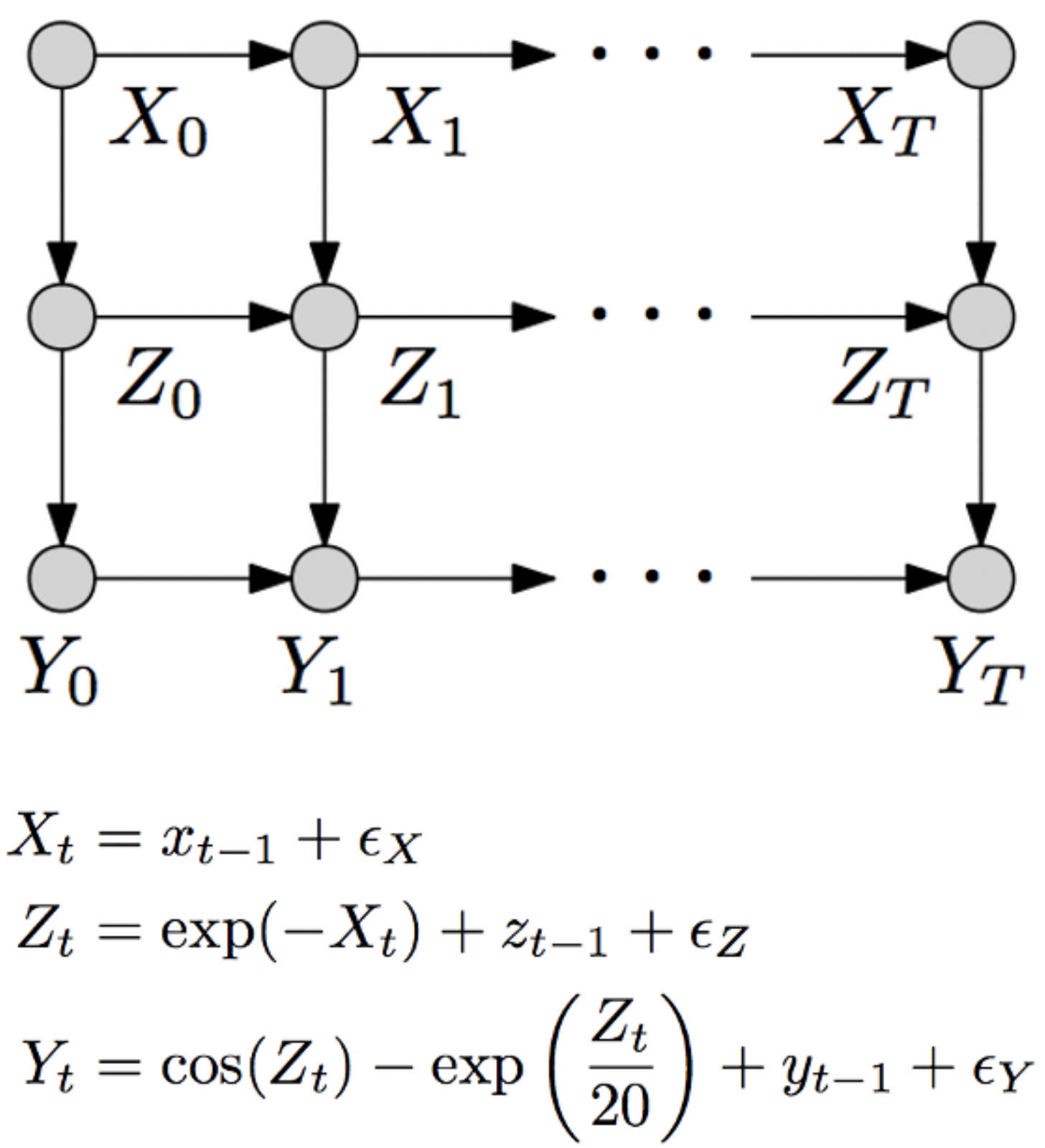}
\end{subfigure}%
\begin{subfigure}{.7\textwidth}
  \centering
  \includegraphics[width=1.\linewidth]{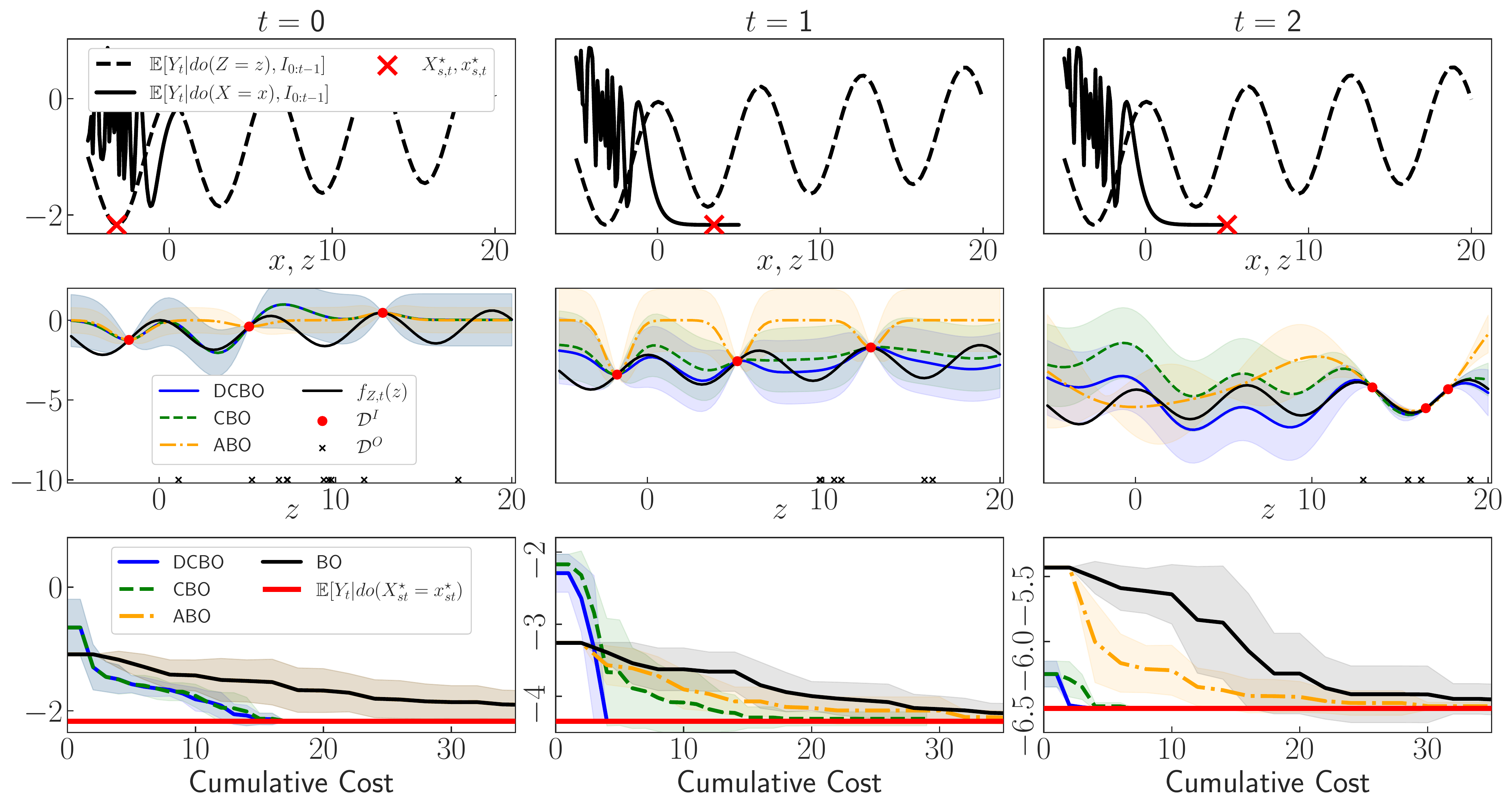}
\end{subfigure}
\caption{Stationary synthetic experiment (\expone). \emph{Left panel}: $\graph_{0:T}$ and \sem. \emph{Right panel, 1\textsuperscript{st} row}: 
Objective functions for the sets in $\missets = \{\{Z\}, \{X\}\}$. 
\emph{Right panel, 2\textsuperscript{nd} row}: Posterior \gptext obtained when using the dynamic causal \gptext construction vs alternative models. 
\emph{Right panel, 3\textsuperscript{rd} row}: Convergence of \our and alternative models to the true optimum (red line) across 10 replicates. Shaded areas give $\pm$ one standard deviation.} 
\label{fig:toy_example}
\end{figure}

\subsection{Noisy manipulative variables (\expnoise)} 
The \sem used for this experiments is given by:
\begin{align*}
    X_t &= X_{t-1}\mathds{1}_{t>0} + \epsilon_X\\
    Z_t &=  \exp(-X_t) + Z_{t-1}\mathds{1}_{t>0} + \epsilon_Z\\
    Y_t &= \text{cos}(Z_t) - \exp(-Z_t/20) + Y_{t-1}\mathds{1}_{t>0} + \epsilon_Y
\end{align*}
where, differently from before, we have $\epsilon_Y \sim \mathcal{N}(0,1)$ and $\epsilon_i \sim \mathcal{N}(2,4)$ for $i \in \{X,Z\}$. We keep the remaining parameters equal to the previous experiment. This means $T=3$, $N=10$, $D(X_t) = \{-5.0, 5.0\}$ and $D(Z_t) = \{-5.0, 20.0\}$. 

\subsection{Missing observational data (\expmissing)} For this experiment we use the same \sem of the experiment \expone However, we set $T=6$, $N=10$ for the first three time steps and $N=0$ afterwards. \fig \ref{fig:missing_exp} shows the convergence paths for this experiment. In this setting \our consistently outperform \cbo at every time step. However, notice how \abo performance improves over time and outperforms \our starting from $t=5$. This is due to the ability of \abo to learn the time dynamic of the objective function and exploit all interventional data collected over time to predict at the next time step. 
\begin{figure*}
\centering
 \includegraphics[width=\linewidth]{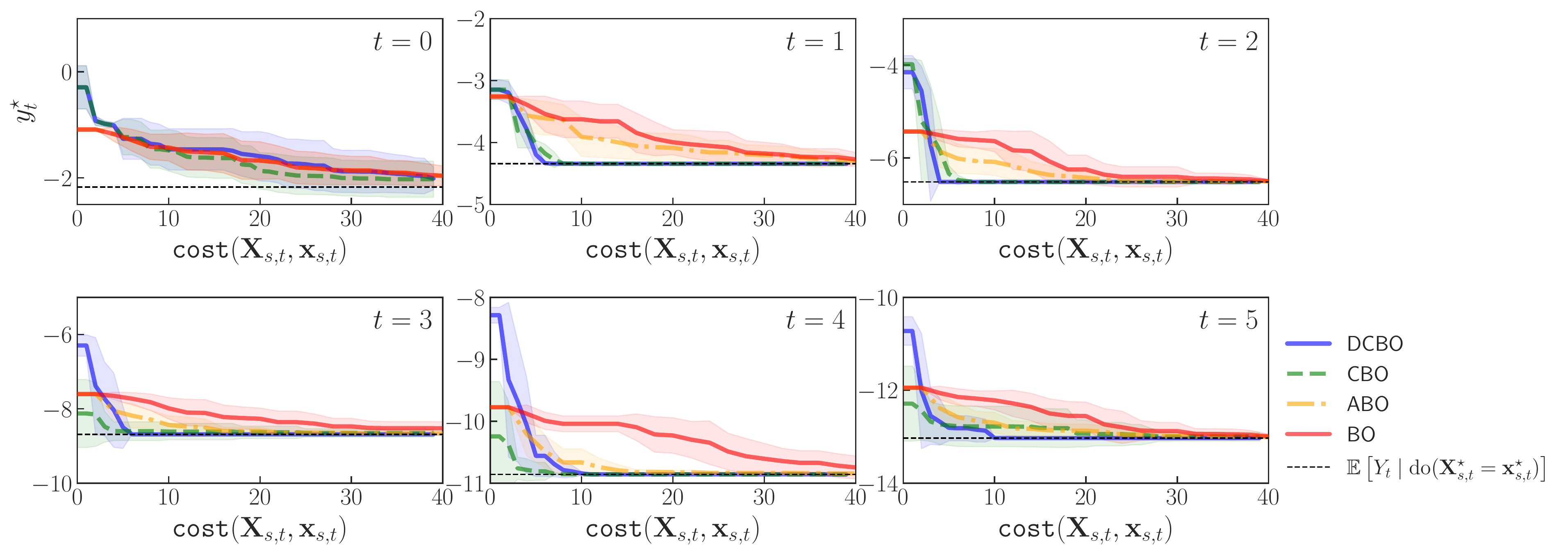}
\caption{Experiment \expmissing. Convergence of \our and competing methods across replicates. The red line gives the optimal $y^*_t, \forall t$. Shaded areas are $\pm$ standard deviation.}
\label{fig:missing_exp}
\end{figure*}

\subsection{Multivariate intervention sets (\expcomplex)} 

The \sem used for this experiments is given by:
\begin{align*}
    W_t &= \epsilon_W\\
    X_t &= - X_{t-1}\mathds{1}_{t>0} + \epsilon_X\\
    Z_t &=  \text{sin}(W_t) - Z_{t-1}\mathds{1}_{t>0} + \epsilon_Z\\
    Y_t &= - 2*\exp(-(X_{t}-1)^2) - \text{exp}(-(X_t + 1)^2)
    - (Z_t - 1)^2 
    \\ & \; \; \; \; - Z_T^2 + \text{cos}(Z_t * Y_{t-1}) - Y_{t-1}\mathds{1}_{t>0} + \epsilon_Y
\end{align*}
where $\epsilon_i \sim \mathcal{N}(0,1)$ for $i \in \{X,Z,W,Y\}$. We set $T=3$, $N=500$, $D(X_t) = \{-5.0, 5.0\}$, $D(Z_t) = \{-5.0, 20.0\}$ and $D(W_t) = \{-3.0, 3.0\}$. Notice that here \our and \cbo explore the set $\missets_t = \{\{X_t\}, \{Z_t\}, \{X_t, Z_t\}\}$ while \bo and \abo intervene on $\{X_t, Z_t, W_t\}$. \fig \ref{fig:complex_exp} shows the convergence paths for this experiment.
\begin{figure*}
\centering
\includegraphics[width=\linewidth]{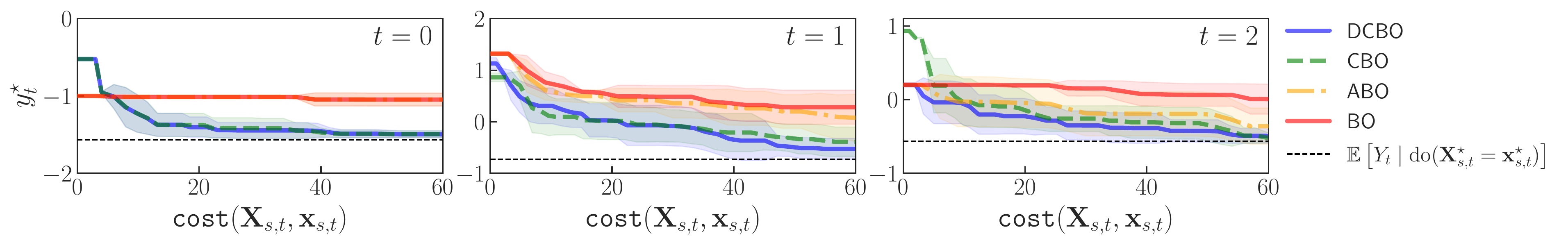}
\caption{Experiment \expcomplex. Convergence of \our and competing methods across replicates. The red line gives the optimal $y^*_t, \forall t$. Shaded areas are $\pm$ standard deviation.}
\label{fig:complex_exp}
\end{figure*}

\subsection{Independent manipulative variables (\expindep)} 

The \sem used for this experiments is given by:
\begin{align*}
    X_t &= - X_{t-1}\mathds{1}_{t>0} + \epsilon_X\\
    Z_t &= - Z_{t-1}\mathds{1}_{t>0} + \epsilon_Z\\
    Y_t &= - 2*\exp(-(X_{t}-1)^2) - \text{exp}(-(X_t + 1)^2)
    - (Z_t - 1)^2 
    \\ & \; \; \; \; - Z_T^2 + \text{cos}(Z_t * Y_{t-1}) - Y_{t-1}\mathds{1}_{t>0} + \epsilon_Y
\end{align*}
where $\epsilon_i \sim \mathcal{N}(0,1)$ for $i \in \{X,Z,Y\}$. We set $T=3$, $N=10$, $D(X_t) = \{-5.0, 5.0\}$ and $D(Z_t) = \{-5.0, 20.0\}$. Notice that here \our and \cbo explore the set $\missets_t = \{\{X_t\}, \{Z_t\}, \{X_t, Z_t\}\}$ while \bo and \abo intervene on $\{X_t, Z_t\}$.  In this case, exploring $\missets_t$ and propagating uncertainty in the causal prior slows down \our convergence, see \cref{fig:indep_exp}.
\begin{figure*}
\centering
\includegraphics[width=\linewidth]{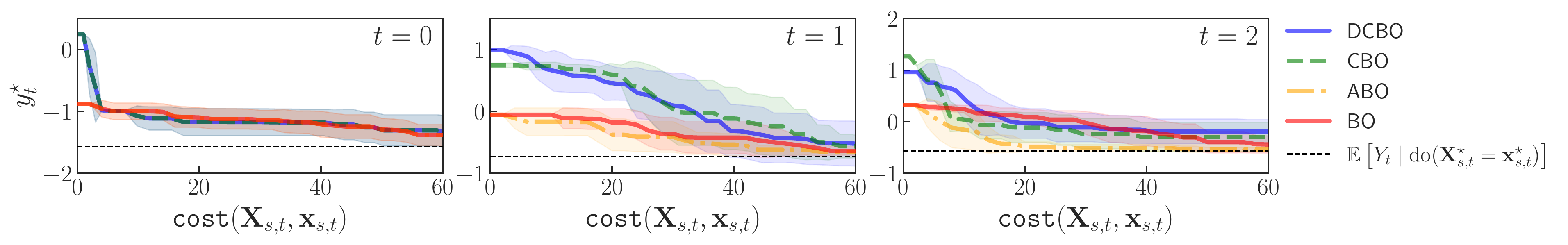}
\caption{Experiment \expindep Convergence of \our and competing methods across replicates. The red line gives the optimal $y^*_t, \forall t$. Shaded areas are $\pm$ standard deviation.}
\label{fig:indep_exp}
\end{figure*}

\subsection{Non-stationary \DAG and \sem (\expnonstat)}

The \sem used for this experiment is more complex than the others due to the fact that the \DAG is non-stationary but so too is the \sem:

\begin{align}
    \begin{cases}
        f(t) &\text{if } t =0 \\
        g(t) &\text{if } t =1 \\
        h(t) &\text{if } t =2
    \end{cases}
\end{align}
where
\begin{align*}
    f(t) = 
    \begin{cases}
    X_t &= \epsilon_X \\
    Z_t &= X_t + \epsilon_Z\\
    Y_t &= \sqrt{|36 - (Z_t-1)^2|} + 1 + \epsilon_Y
    \end{cases}
\end{align*}
\begin{align*}
    g(t) = 
    \begin{cases}
    X_t &= X_{t-1} + \epsilon_X \\
    Z_t &= -\frac{X_t}{X_{t-1}} + Z_{t-1} + \epsilon_Z \\
    Y_t &= Z_t\cos(Z_t \pi) - Y_{t-1} + \epsilon_Y
    \end{cases}
\end{align*}
\begin{align*}
    h(t) = 
    \begin{cases}
    X_t &= X_{t-1} + \epsilon_X\\
    Z_t &= X_t + Z_{t-1} + \epsilon_Z\\
    Y_t &= Z_t - Y_{t-1} - Z_{t-1} + \epsilon_Y
    \end{cases}
\end{align*}
where $\epsilon_i \sim \mathcal{N}(0,1)$ for $i \in \{X,Z,Y\}$. We set $T=3$, $N=10$, $D(X_t) = \{-5.0, 5.0\}$ and $D(Z_t) = \{-5.0, 20.0\}$. Notice that here \our and \cbo explore the set $\missets_t = \{\{X_t\}, \{Z_t\}, \{X_t, Z_t\}\}$ while \bo and \abo intervene on $\{X_t, Z_t\}$.

\subsection{Real-World Economic data (\exprealec)} We create an observational data set by extracting the following indicators from the \acro{oecd} data portal (\texttt{https://data.oecd.org/}):

\begin{itemize}
    \item \acro{gdp} = \acro{gdp} in milion of US dollars.
    \item \acro{cpi} = annual growth of inflation measured by consumer price index \acro{cpi}.
    \item \acro{taxrev} = tax revenues measured as a percentage of \acro{gdp}.
    \item \acro{hur} = unemployment rate as measured by the numbers of unemployed people as a percentage of the labour force. 
\end{itemize}

We manipulate these indicators to get the nodes in the \DAG of \fig. \ref{fig:dag_econ}. We define 
\begin{align*}
    U_t &= \log(\acro{hur}_t) \\
    T_t &= \frac{\acro{taxrev}_t*\acro{gdp}_t - \acro{taxrev}_{t-1}*\acro{gdp}_{t-1}}{\acro{taxrev}_{t-1}*\acro{gdp}_{t-1}} \\
    G_t &= \frac{\acro{gdp}_t - \acro{gdp}_{t-1}}{\acro{gdp}_{t-1}} \\
    I_t &= \acro{cpi}_t
\end{align*}
For this analysis we consider the annual data for 10 countries namely Australia, Canada, France, Germany, Italy, Japan, Korea, Mexico, Turkey, Great Britain and the United States of America for the period (2000 - 2019). We fit the following \sem:
\begin{align*}
    T_t &= f_T(t) + \epsilon_T \\
    I_t &= f_I(t) + \epsilon_I \\
    G_t &= f_G(T_t, I_t) + \epsilon_G \\
    U_t &= f_U(G_t, I_t) + \epsilon_U \\
\end{align*}
by placing \gptext{s} on all functions $f_i(\cdot), i \in \{T, I, G, U\}$. This \sem is then used to generate interventional data and compute the values of $y^\star_t, t = 2010,\dots, 2012$. 

\begin{figure}[ht!]
\centering
\includegraphics[width=\linewidth]{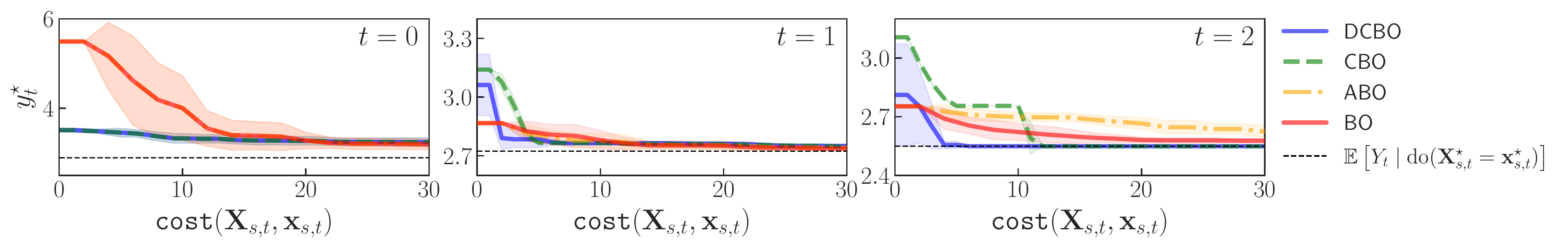}
\caption{Experiment \exprealec Convergence of \our and competing methods across replicates. The black line gives the optimal $y^*_t, \forall t$. Shaded areas are $\pm$ one standard deviation.}
\label{fig:economy_example}
\end{figure}

We run the optimization 10 times and plot the convergence path for \our and competing models (see \fig \ref{fig:economy_example}). While all method perform similarly at $t=2010$ and $t=2011$, \our outperforms competing approaches at $t=2012$. On average (see \cref{tab:table_gaps}) \our finds the optimal intervention faster.

\subsection{Results without convergence}
We repeat all experiments in the paper allowing the algorithms to perform a lower number of trials at every time steps. This means that, for $t>0$, when moving to step $t$ the convergence of the algorithm at step $t-1$ is not guaranteed. In turn this affect the optimum value that the algorithm can reach at subsequent steps. Results are given in \cref{tab:table_gaps_noconv} and \cref{tab:table_sets_noconv}. The convergence paths for \our and competing methods are given in \cref{fig:static_example_noconv} to  \cref{fig:multivariate_example_noconv}.

\begin{table}[htbp]
\centering
\caption{Average modified gap measure (10 replicates) across time steps and for different experiments. See \fig \ref{fig:map_methods} for a summary of the compared methods. Higher values are better. The best result for each experiment is bolded. Standard errors in brackets.}
\label{tab:table_gaps_noconv}
\resizebox{\columnwidth}{!}{
%@{\extracolsep{4pt}}@{\kern\tabcolsep}
\begin{tabular}{lcccccccc}
\toprule
&  \multicolumn{6}{c}{Synthetic data} & \multicolumn{2}{c}{Real data} \\
\cmidrule(lr){2-7} 
\cmidrule(lr){8-9}
& \expone & \expmissing & \expnoise &  \expcomplex & \expindep & \expnonstat & \exprealec & \exprealpol \\
\cmidrule(lr){2-7} 
\cmidrule(lr){8-9}
\multirow{2}{*}{$\our$} &  \textbf{0.88} & \textbf{0.72}  & \textbf{0.73} & \textbf{0.49} & 0.47 & \textbf{0.47} & 0.40 &  \textbf{0.67}\\
          &  (0.00) & (0.07) & (0.00) & (0.00) & (0.05) & (0.00) & (0.04) & (0.00) \\
\multirow{2}{*}{$\cbo$} & 0.57 & 0.51 & 0.67 & 0.47 & 0.48 & \textbf{0.47} & \textbf{0.41} & 0.65 \\
          & (0.02) & (0.09) & (0.01) & (0.04) & (0.04) & (0.00) & (0.04) & (0.00)\\
\multirow{2}{*}{$\abo$} & 0.43 & 0.45 & 0.42 & 0.40 & \textbf{0.50} & 0.41 & 0.38 & 0.47\\
          & (0.06) & (0.04) & (0.06) & (0.05) & (0.00) & (0.03) & (0.04) & (0.01) \\
\multirow{2}{*}{$\bo$} & 0.42 & 0.41 & 0.41 & 0.38 & \textbf{0.50} & 0.40& 0.40 & 0.46\\
          & (0.06) & (0.05) & (0.07) & (0.07) & (0.01) & (0.04) & (0.04) & (0.03)\\
\bottomrule
\end{tabular}
}
\end{table}

\begin{table}[htbp]
\centering
\caption{Average percentage of replicates across time steps and for different experiments for which the optimal intervention set is identified. See \fig \ref{fig:map_methods} for a summary of the compared methods. Higher values are better. The best result for each experiment is bolded.}
\label{tab:table_sets_noconv}
\resizebox{\columnwidth}{!}{
\begin{tabular}{lcccccccc}
\toprule
&  \multicolumn{6}{c}{Synthetic data} & \multicolumn{2}{c}{Real data} \\
\cmidrule(lr){2-7} 
\cmidrule(lr){8-9}
& \expone & \expmissing & \expnoise &  \expcomplex & \expindep & \expnonstat & \exprealec & \exprealpol \\
\cmidrule(lr){2-7} 
\cmidrule(lr){8-9}
\multirow{1}{*}{$\our$} & \textbf{90.0} & \textbf{70.00} & \textbf{93.00}  & \textbf{93.33} & \textbf{96.67} & \textbf{66.67} & 73.33 & \textbf{33.33} \\
\multirow{1}{*}{$\cbo$} & 76.67 & 63.33 & 76.67 & 86.67 & 93.33 & 33.33 &   \textbf{80.00} & \textbf{33.33}\\
\multirow{1}{*}{$\abo$} & 0.00 & 0.00 & 0.00 & 0.00 & 100.00 & 0.00 & 66.67 & 0.00\\
\multirow{1}{*}{$\bo$} & 0.00 & 0.00 & 0.00 & 0.00 & 100.00 & 0.00 & 66.67 & 0.00\\
\bottomrule
\end{tabular}
}
\end{table}

\begin{figure}[ht!]
\centering
\includegraphics[width=\linewidth]{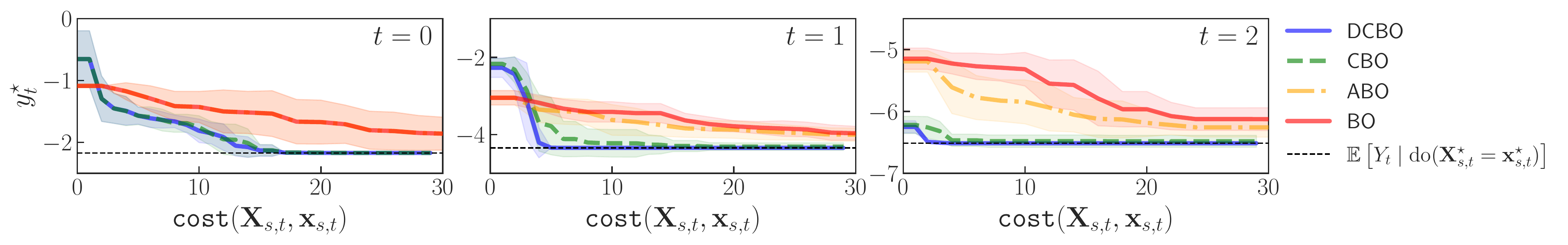}
\caption{Experiment \expone with maximum number of trials $H=30$. Convergence of \our and competing methods across replicates. The black line gives the optimal $y^*_t, \forall t$. Shaded areas are $\pm$ one standard deviation.}
\label{fig:static_example_noconv}
\end{figure}

\begin{figure}[ht!]
\centering
\includegraphics[width=\linewidth]{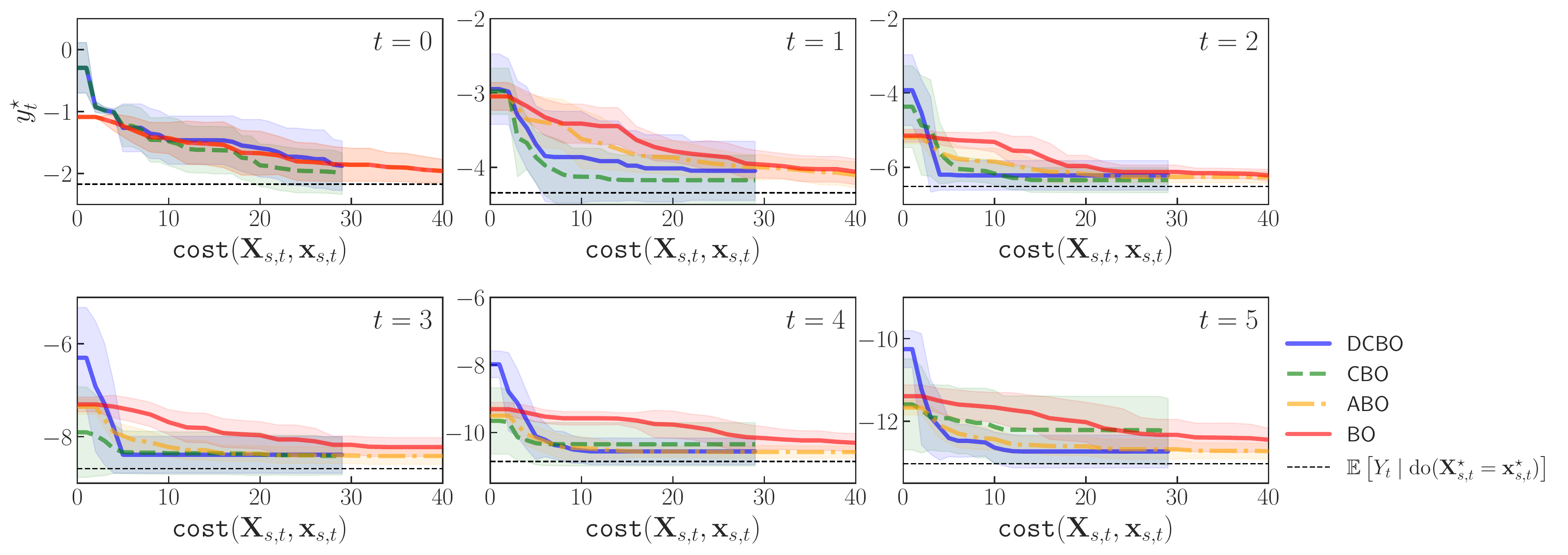}
\caption{Experiment \expmissing with maximum number of trials $H=30$. Convergence of \our and competing methods across replicates. The black line gives the optimal $y^*_t, \forall t$. Shaded areas are $\pm$ one standard deviation.}
\label{fig:missing_example_noconv}
\end{figure}

\begin{figure}[ht!]
\centering
\includegraphics[width=\linewidth]{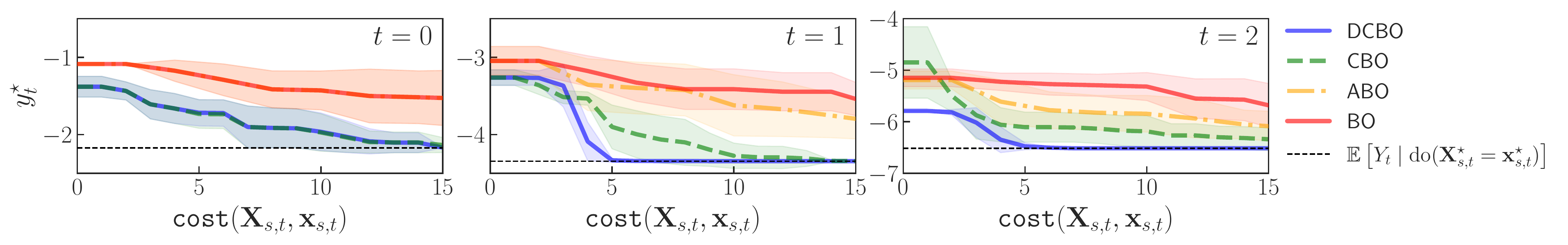}
\caption{Experiment \expnoise. with maximum number of trials $H=30$. Convergence of \our and competing methods across replicates. The black line gives the optimal $y^*_t, \forall t$. Shaded areas are $\pm$ one standard deviation.}
\label{fig:noisy_example_noconv}
\end{figure}

\begin{figure}[ht!]
\centering
\includegraphics[width=\linewidth]{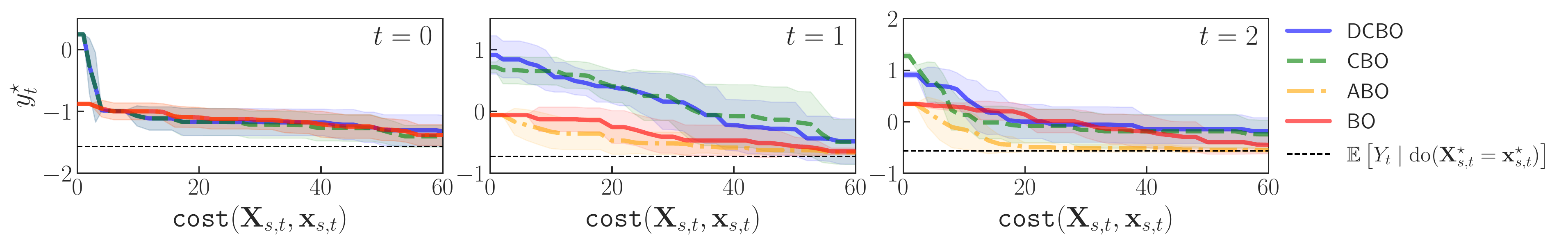}
\caption{Experiment \expindep with maximum number of trials $H=30$. Convergence of \our and competing methods across replicates. The black line gives the optimal $y^*_t, \forall t$. Shaded areas are $\pm$ one standard deviation.}
\label{fig:independent_example_noconv}
\end{figure}

\begin{figure}[ht!]
\centering
\includegraphics[width=\linewidth]{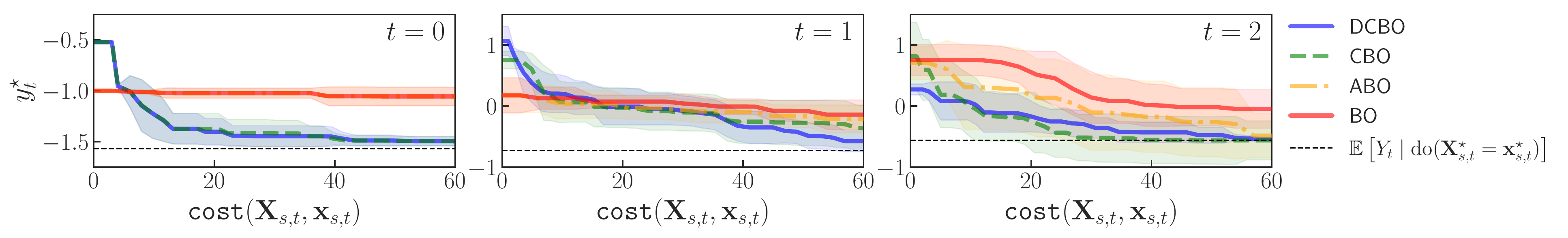}
\caption{Experiment \expcomplex with maximum number of trials $H=30$. Convergence of \our and competing methods across replicates. The black line gives the optimal $y^*_t, \forall t$. Shaded areas are $\pm$ one standard deviation.}
\label{fig:multivariate_example_noconv}
\end{figure}

%\clearpage
%\newpage

\subsection{Results over multiple datasets and replicates}
In this section we show the results obtained running all the experiments in the main paper across 10 different observational dataset sampled from the \sem given above. Results are given in \cref{tab:table_gaps_10obs}.

\begin{table}[ht!]
\centering
\caption{Average modified gap measure \textbf{across 10 observational datasets and 10 replicates}. Results are average figures across time steps. See \cref{fig:map_methods} for a summary of the compared methods. Higher values are better. The best result for each experiment is bolded. Standard errors in brackets.}

\label{tab:table_gaps_10obs}
\begin{tabular}{lcccccc}
\toprule
&  \multicolumn{6}{c}{Synthetic data} \\
\cmidrule(lr){2-7} 
& \expone & \expmissing & \expnoise &  \expcomplex & \expindep & \expnonstat \\
\cmidrule(lr){2-7} 
\multirow{2}{*}{$\our$} &  \textbf{0.83} & \textbf{0.82} & \textbf{0.82} & \textbf{0.48} & 0.46 & 0.63\\
          &  (0.06) & (0.05) & (0.05) & (0.02) & (0.03) & (0.06)\\
\multirow{2}{*}{$\cbo$} & 0.80 & 0.68 & 0.74 & \textbf{0.48} & 0.47 & \textbf{0.64} \\
          & (0.05) &  (0.04) & (0.09)  & (0.01) & (0.02) & (0.04)\\
\multirow{2}{*}{$\abo$} & 0.47 & 0.49 & 0.47 & 0.45 & 0.48 & 0.38\\
          &  (0.01) &  (0.00) & (0.01)  & (0.08) & (0.00)& (0.01)\\
\multirow{2}{*}{$\bo$} & 0.47 & 0.47 & 0.47 & 0.40 & \textbf{0.50} & 0.38\\
          & (0.01) &  (0.01) & (0.01)  & (0.07) & (0.00)& (0.01)\\
\bottomrule
\end{tabular}
\end{table}

\clearpage

\section{Intervening on differential equations (\exprealpol)}
\label{sec:ode_details}

In this section we describe in detail the experiment conducted in \cref{sec:real_experiments}. This example is based on the work by \citet{blasius2020long}. In this demonstration we continue along that paradigm when we investigate a biological systems in which two species interact, one as a predator and the other as prey. \citet{blasius2020long} performed microcosm experiments (in a chemostat or bioreactor) with a planktonic predator–prey system.

We use the provided ODE (\cref{sec:ode_blasius}) from the paper \citep[Methods]{blasius2020long}, which describes a stage-structured predator–prey community in a chemostat, as our \sem. As $\dataset^O$ we use the experimental data collected in vitro (for raw data see supplementary material of \citep{blasius2020long}). The corresponding \DAG (\cref{sec:plankton_dag}) and \sem (\cref{sec:ode_as_sem}) is constructed from the ODE (see overleaf), by rolling out the temporal variable dependencies in the ODE (the idea is well illustrated in \citep[Fig. 1]{weber2016incompatibility}).

Using this setup we investigate a requisite intervention policy necessary to reduce the concentration of dead animals in the chemostat -- $D_t$ in \cref{fig:dag_ode}.

\subsection{Interpreting differential equations as causal models}

A lot of work \citep{peters2020causal, kaiser2016limits, mooij2013uai, bongers2018random, hansen2014causal, weber2016incompatibility} has been dedicated to interpreting ordinary differential equations as structural causal models and consequently the associated task of intervening therein. More precisely, attention has been placed on extending causal theory \citep{pearl2000causality, spirtes1995dag} to the cyclic case, thereby enabling causal modelling of systems that involve feedback \citep{mooij2013uai,koster1996markov,dechter1996identifying,Neal_2000,hyttinen2012learning,rubenstein2016deterministic,peters2020causal}.

Naively, the simplest extension to the cyclical case is by simply dropping the acyclicity constraint from the \sem \citep[\S 1]{mooij2013uai}. But then we are faced with a new problem: how do we ``interpret cyclic structural equations'' \citep{mooij2013uai}? The most common approach is to ``assume an underlying discrete-time dynamical system, in which the structural equations are used as fixed point equations'' \citep{mooij2013uai}. This renders a simple schema wherein which we use the \sem as a set of updates rules, to find the values of the variables at $t+1$, using the information from $t$. This is a popular paradigm, advanced by e.g. \citet{spirtes1995dag,hyttinen2012learning, dash2005restructuring, lacerda2012discovering, mooij2011causal}. This is also the one we will use herein.

Another philosophy that deals with interventions in systems, was developed by \citet{casini2011models}. In the same vein is the work by \citet{gebharter2014formal,gebharter2016modeling}. This suite of work comes from the philosophy of science domain, rather than the statistical and machine learning literature, briefly reviewed in the previous two paragraphs. Theirs is primarily a concern with mechanisms (specifically ``mechanistic biological models with complex dynamics'' in the case of \citet{kaiser2016limits}) -- fundamentally they are the same thing as our causal effects but the perspective is different. \citet{casini2011models} suggests that modelling (acyclical) mechanisms should be done by way of recursive Bayesian networks (RBN). \citet{gebharter2014formal} points out some shortcomings with Caisini's approach and proposes the multilevel causal model (MLCM) as a remedy. Notably though, both works assume acyclicity (and so cannot feature mechanisms with feedback) of the problem domain a shortcoming that \citet{gebharter2016modeling} deals with by extending the MLCM to allow for cycles. For completeness we should also say that the RBN was extended to handle cycles by \citet{clarke2014modelling} (their approach was used \citet{gebharter2016modeling} for extending the MLCM).

\subsection{Ordinary differential equation}
\label{sec:ode_blasius}

\citet{blasius2020long} develop a mathematical model, the set of ordinary differential equations in \cref{eq:ode_N}--\cref{eq:ode_D}, to describe a stage-structured predator–prey community in a chemostat, which closely follows their experimental setup.

\begin{align}
    \frac{\dint N}{ \dint t} &= \delta N_{\textrm{in}} - F_P (N) P - \delta N \label{eq:ode_N}\\
    \frac{\dint P}{ \dint t} &= F_P (N) P - \frac{ F_B (P) B}{\varepsilon} - \delta P \\
    \frac{\dint E}{ \dint t} &= R_E - R_J - \delta E \\ 
    \frac{\dint J}{ \dint t} &= R_J - R_A - (m+\delta)J \\ 
    \frac{\dint A}{ \dint t} &= \beta R_A - (m+\delta)A \\
    \frac{\dint D}{ \dint t} &= m(J+A) - \delta D \label{eq:ode_D}
\end{align}

A full description of all variables and parameters can be found in \cref{table:ode}.
\begin{table}[htbp]
\centering
\caption{Table describing variable and parameters of ODE in \cref{eq:ode_N} -- \cref{eq:ode_D}.\label{table:ode}}
\begin{tabular}{clcc}
\toprule
Variable & Description & Value & Unit \\
\midrule
$N$ & Nitrogen (prey) concentration & $\in \mathbb{R}^1$ & $\unit$ \\ 
$P$ & Phytoplankton (predator) concentration & $\in \mathbb{R}^1$ & $\unit$ \\ 
$E$ & Predator egg concentration & $\in \mathbb{R}^1$ & $\unit$ \\ 
$J$ & Predator juvenile concentration & $\in \mathbb{R}^1$ & $\unit$ \\ 
$A$ & Predator adult concentration & $\in \mathbb{R}^1$ & $\unit$ \\ 
$D$ & Dead animal concentration & $\in \mathbb{R}^1$ & $\unit$ \\
\midrule
Parameter & Description & Value & Unit \\
\midrule
$N_{\text{in}}$ & Nitrogen concentration in the external medium & 80 & $\unit$ \\
$F_P$ & Algal nutrient uptake & - & $\unit$ \\
$F_B$ & Rotifer nutrient uptake & - & $\unit$\\
$\varepsilon$ & Predator assimilation efficiency & 0.55 & - \\
$R_E$ & Egg recruitment rate & - & -\\
$R_J$ & Juvenile recruitment rate & - & -\\
$R_A$ & Adult recruitment rate & - & -\\
$m$ & Rotifer (predator) mortality rate & 0.15 & Per day\\
$\beta$ & Adult/juvenile mass ratio & 5 & - \\
\bottomrule
\end{tabular}
\end{table}

For additional details see \citep[Methods]{blasius2020long}.

\subsection{Corresponding directed acyclical graph}
\label{sec:plankton_dag}

The original rolled-out \DAG (\cref{fig:ode_approx_1}) is modified to remove graph cycles (\cref{fig:ode_approx_2}), where the corresponding dependencies are replicated in the \sem. Now, note first that the temporal roll-out of \cref{fig:ode_dependencies} contains no cycles (once the self-cycles have been re-purposed as temporal transition functions). Nonetheless, comparing \cref{fig:ode_approx_1} and \cref{fig:ode_approx_2} it can be seen that two edges have been removed to simplify the causal dependencies on the phytoplankton (predator) concentration i.e. to make it only dependent on the nitrogen concentration in the external medium as well as the most immediate predator concentration at time $t-1$.

\begin{figure*}[ht!]
    \centering
    \begin{subfigure}[t]{0.2\textwidth}
        \centering
        \includegraphics[width=0.6\textwidth]{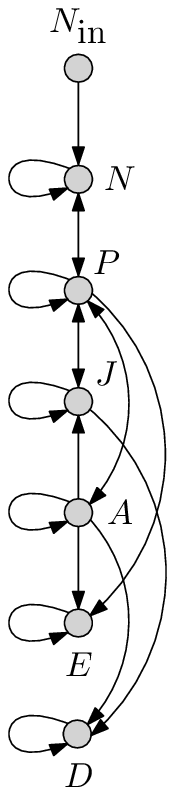}
        \caption{ODE variable dependencies.\label{fig:ode_dependencies}}
    \end{subfigure}%
    \hspace*{\fill}
    \begin{subfigure}[t]{0.39\textwidth}
        \centering
        \includegraphics[width=\textwidth]{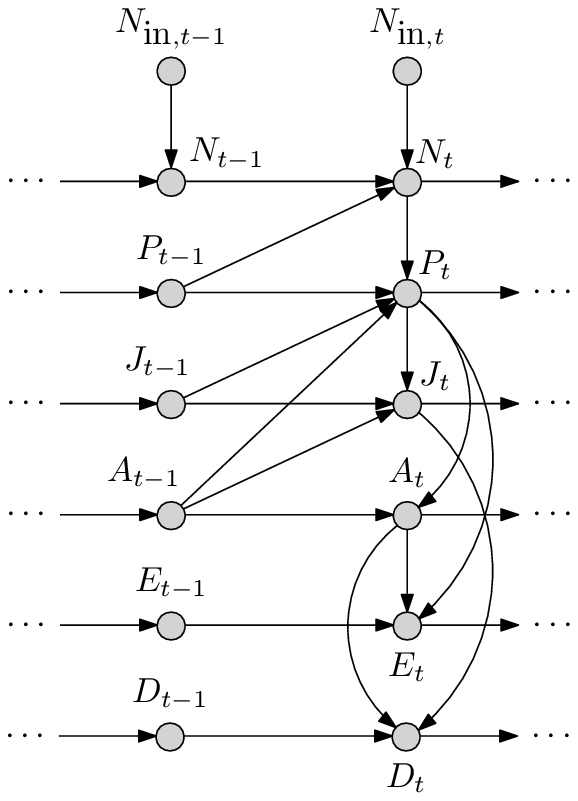}
        \caption{First DAG approximation. \label{fig:ode_approx_1}}
    \end{subfigure}%
    \hspace*{\fill}
    \begin{subfigure}[t]{0.39\textwidth}
        \centering
        \includegraphics[width=\textwidth]{figures/ode/second_approximation.eps}
        \caption{Second DAG approximation. \label{fig:ode_approx_2}}
    \end{subfigure}%
    \caption{Proposed time-indexed DAG representing the causal dependencies in the stage-structured predator–prey community in a chemostat. The vertices of the graph represent the concentrations of the different chemostat compounds at different discrete time points, where time is moving from left to right. \cref{fig:ode_dependencies} shows the variable dependencies as described in the original ODE found in \cref{eq:ode_N} -- \cref{eq:ode_D} -- notice the presence of self-loops and cycles. \Cref{fig:ode_approx_1} shows a first approximation to a corresponding causal graph, where the ODE has been `rolled' out in time -- note the absence of self-loops and cycles. \Cref{fig:ode_approx_2} shows a second approximation to the original ODE dynamics but this time removing two parent dependencies from $P_t$.}
\end{figure*}

One large deviation from the original set of ODEs, is that we treat the $N_{\text{in}}$ as an instrument variable and moreover allow it to be manipulative. This means that in order to reduce the concentration $D_t$ we allow the optimisation frameworks to intervene also on $N_{\text{in}}$.

\subsection{ODE as SEM}
\label{sec:ode_as_sem}

We fit the following \sem, based on the \DAG in \cref{fig:ode_approx_2}:
\begin{align}
    \label{eq:ODE_as_SEM}
    N_{\text{in},t} &= \epsilon_{N_{\text{in}}} \\
    N_t &= f_N(N_{\textrm{in}, t}, N_{t-1},  P_{t-1}) + \epsilon_N \\
    P_t &= f_P(N_{t}, P_{t-1}) + \epsilon_P \\
    J_t &= f_J(P_{t}, J_{t-1}, A_{t-1}) + \epsilon_J \\
    A_t &= f_A(P_t, A_{t-1}) + \epsilon_A \\
    E_t &= f_E(P_t, A_t, E_{t-1}) + \epsilon_E \\
    D_t &= f_D(J_t, A_t, D_{t-1}) + \epsilon_D
\end{align}

by placing \gptext{s} on all functions $\{f_i(\cdot) \mid i \in \{N_{\text{in}}, N,P,E,J,A,D\} \}$. This \sem is then used to generate interventional data and compute the values of $\{d^\star_t \mid t = 0,1,2\}$.

Further, $\{\epsilon_j \sim \mathcal{N}(0,1) \mid j \in \{N_{\text{in}}, N,P,E,J,A,D\} \} $. We set $T=3$, $N=4$ where the manipulative variables are: $N_{\text{in},t}, J_t$ and $A_t$. This means in practise that we are interested in the start of the simulation where we are trying to reduce the mortality concentration, in the chemostat, from beginning where our observational samples $\mathcal{D}^{O}$ are formed from four time-series\footnote{We use data-files \texttt{C1.csv, C2.csv, C3.csv, C4.csv} from the original publication \citep{blasius2020long} -- available here: \url{https://figshare.com/articles/dataset/Time_series_of_long-term_experimental_predator-prey_cycles/10045976/1} [Accessed: 01/04/21].}. 

Intervention domains are given by
\begin{align*}
    D(N_{\text{in},t}) &= [40.0,160.0] \\ 
    D(J_t) &= [0.0, 20.0] \\
    D(A_t) &= [0.0, 100.0]
\end{align*}
Notice that \our and \cbo explore the set 
\begin{equation*}
\missets_t = \{\{N_{\text{in},t}\}, \{J_t\}, \{A_t\}, \{N_{\text{in},t}, J_t\}, \{N_{\text{in},t}, A_t\}, \{J_t, A_t\}, \{N_{\text{in},t}, J_t, A_t\}\}    
\end{equation*}
while \bo and \abo will only intervene on $\{N_{\text{in},t}, J_t, A_t\}$. The optimal sequence of interventions is given by $\{\{J_0, A_0\}, \{M_1\}, \{M_2\}\}$. 

Results are shown in \cref{fig:ode_converged}. Note that the performance of \our and \cbo are almost identical.

\begin{figure}[ht!]
\centering
\includegraphics[width=\linewidth]{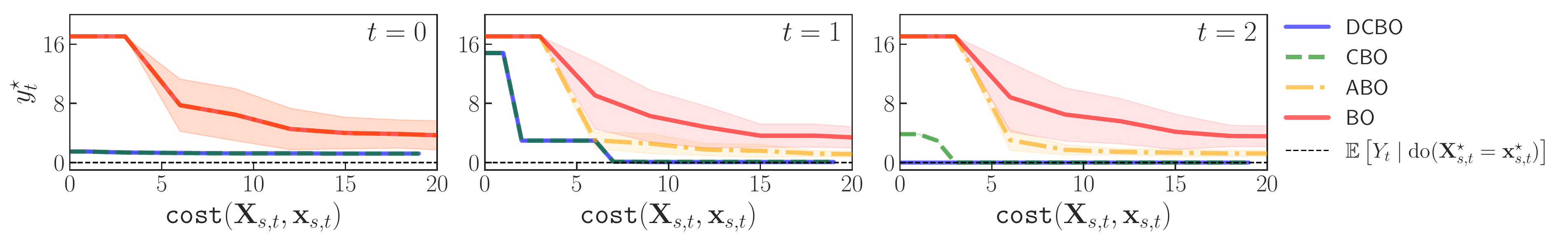}
\caption{Experiment \exprealpol with maximum number of trials $H=20$. Convergence of \our and competing methods across replicates. The black line gives the optimal $y^*_t, \forall t$. Shaded areas are $\pm$ one standard deviation.}
\label{fig:ode_converged}
\end{figure}

% \begin{figure}[ht!]
% \centering
% \includegraphics[width=\linewidth]{figures/experimental_results/ode_no_blanket_10052021_1556.pdf}
% \caption{Experiment \exprealpol with maximum number of trials $H=20$. Convergence of \our and competing methods across replicates. The black line gives the optimal $y^*_t, \forall t$. Shaded areas are $\pm$ one standard deviation.}
% \label{fig:ode_not_converged}
% \end{figure}

% \subsection{Experimental observations}

% \nd{Maybe make some data plots.}

\newpage

\section{Applicability of \our to real-world problems}\label{sec:appl_rw}

As previously done in \cbo \citep{cbo} and other causal decision-making frameworks (\eg \cite{bareinboim2015bandits}) for static settings, in \our we assume to be able to repeatedly intervene in the system with interventions that have an instantaneous effect observed within the time slice duration. In other words, within every time step, we perform an intervention that changes the system and that leads to an effect for which we collect the corresponding target experimental value. However, the system reverts back once the experiment has been implemented and the agent can then explore alternative interventions and measure their effect too. In \our, the dynamics of the time resolution specified by the graph time indices is slower than the time you can take actions and see the effects.

While this assumption can be difficult to verify when interacting directly with the psychical world, it does not limit the applicability of the proposed framework to real-world problems. Indeed, in a variety of real-world settings, simulators or digital twins of real-world assets/processes are used in industrial settings and are fundamental in selecting actions before intervening in the real physical world. Digital twins provide virtual replicas of a physical object or system, such as a bridge or an engine, that engineers use for simulations before something is created or to monitor its operation in real-time. Examples are given by the digital twin of a 3\acro{D}-printed stainless steel bridge \citep{bridge}, \acro{nasa} and \acro{u.s.} Air Force vehicles \citep{glaessgen2012digital}, jet-engine monitoring, infrastructure inspection as well as cardiac medicine \citep{virtual_replica}. In all these settings, observational data are used to build the emulator which is ``a living computer model which is continuously learning to imitate the physical world'' \citep{bridge}. We can then intervene on the digital twin to collect interventional data and measure the causal effects. Intervening in a simulator has a cost e.g. a computationally cost thus interventions need to be carefully picked by employing a probabilistic model that correctly quantifies uncertainty and integrates different sources of information. In \our this is done by using the dynamic causal \gptext model. Once an intervention has been implemented, the digital twin ``reverts'' to its unperturbed/observational nature (i.e. without intervention), allowing the user to investigate other interventions without having changed the ``underlying state of the system'' nor, indeed, the true system. Once an optimal intervention is found, the agent can implement it in the real system thus changing it. Note that our approach allows for noise in the likelihood function thus the simulator can be a noisy version of the physical world. %In this paper we focus on demonstrating how one can take decisions once a simulator has been created resorting to domain knowledge. Future collaborations with practitioners across different fields will enable further developing of the \our framework.

\section{Connections}
\label{sec:connections}
We conclude by providing a discussion of the links between \our, the two methodologies used as benchmarks in the experimental session, namely the \cbo algorithm \citep{cbo} and the \abo algorithm \citep{abo}, and the literature on bandits and \rl. We discuss how their problem setups differ from our and highlight the reasons why \our is needed to solve the problem in \eq \eqref{eq:dcgo}. 

\paragraph{\cbo algoritm}  
The \cbo algorithm \citep{cbo} can be used to find optimal interventions to perform in a causal graph so as to optimize a single target node $Y$. \cbo addresses static settings where variables in $\graph$ are i.i.d. across time steps, \ie $p(\mat{V}_t) = p(\mat{V})\text{,}\forall t$, and only one static target variable exists. For instance, \cbo can be used to find the optimal intervention for $Y$ in the \DAG of \fig \ref{fig:map_methods}b. 
In order to use \cbo for the \DAG of \fig \ref{fig:map_methods}a, one would need to identify a unique target among $Y_{0:T}$, \eg $Y_T$. However, optimizing $Y_T$ might lead to chose interventions that are sub-optimal for $Y_{0:T-1}$ thus not solving the problem in \eq \eqref{eq:dcgo}. In addition, to find the optimal intervention for $Y_T$, \cbo explores all interventions in $\mathcal{P}(\mat{X}_{0:T})$ which results in a large search space and requires performing a high number of interventions. This slows down the convergence of the algorithm and increases the optimization cost. %leading to higher cost . %In a variety of application this is not possible. 
%Exploring a large search space would also significantly slow down the algorithm (see experiments X). 
One can alternatively run \cbo $T$ times optimizing $\Yt$ at each time step. Doing that would require re-initializing the surrogate models for the objective functions at every $t$ and would thus imply loosing all the information collected from previous interventions. Finally, in optimizing $\Yt$, \cbo does not account for how the previously taken interventions have changed the system again slowing down the convergence of the algorithm. In order to recursively optimise intermediate outputs given the previously taken decisions one need to resort to \our. By changing the objective function at every time step, incorporating prior interventional information in the objective function and limiting the search space at every time step based on the topology of the $\graph$, \our addresses the \cbo issues mentioned above making it a framework that can be practically used for sequential decision making in a variety of applications.

\paragraph{\abo algorithm} 
While \cbo tackles the causal dimension of the \dgo problem but not the temporal dimension, the \abo algorithm also addresses dynamic settings but does not account for the causal relationships among variables, see \fig \ref{fig:map_methods} for a graphical representation of the relationship between these methods. As in \bo, \abo finds the optimal intervention values by breaking the causal dependencies between the inputs and intervening simultaneously on all of them thus setting $\Xst = \Xt$ for all $t$. Additionally, considering the inputs as fixed and not as random variables, \abo does not account for their temporal evolution. This is reflected in the \DAG of \fig \ref{fig:map_methods}(c) where both the horizontal links between the inputs and the edges amongst the input variables are missing. In solving the problem in \eq \eqref{eq:dcgo} for the \DAG in \fig \ref{fig:map_methods}a, \bo would disregard both the temporal dependencies in $Y$ and the input dependencies (\DAG in \fig \ref{fig:map_methods}d) while \abo would keep the former but ignore the latter. Differently from our approach, \abo considers a continuous time space and places a surrogate model on $\Yt = f(\x, t)$. $f(\x, t)$ is then modelled via a spatio-temporal \gptext with separable kernel. The \abo acquisition function for $f(\x, t)$ is then restricted to avoid collecting points in the past or too far ahead in the future where the \gptext predictions have high uncertainty. The spatio-temporal \gptext allows \abo to predict the objective function ahead in time and track the evolution of the optimum. However, in order for \abo to work the objective function rate of change over time must be slow enough to gather enough samples to learn the relationships in space and time. In our discrete time setting this condition is equivalent to ask that, at every time step, it is possible to perform different interventions with an underlying true function that does not change.  
%collect enough sample to find within each time step we can intervene $P$ times and the function will be the same withing t. 
Note that also in \our, Assumptions \ref{assumptions} imply a certain level of regularity in the objective functions. For instance, in the \DAG of \fig \ref{fig:map_methods}a, given that $\pa{Y_t} = \{Z_t, Y_{t-1}\}, \forall t > 0$, the objective functions have a constant shape and are only shifted vertically by the performed interventions. While some regularity is also required in \our, through the causal graph we impose more structure on the objective function and its input thus lowering the need for exploration. The more accurate the estimation of the functions in the \sem is the more we can track the dynamic of the objective function and we can deal with sharp changes in the objectives.  One additional important difference between \abo and \our is in the exploration of different intervention set. Indeed, by intervening on all variables, \abo can lead to sub-optimal solution. As mentioned for \bo in \cite{cbo}, depending on the structural relationships between variables, intervening on a subgroup might lead to a propagation of effects in the causal graph and a higher final target. In addition, intervening on all variables is cost-ineffective in cases when the same target can be obtained by setting only a subgroup of them. This is particularly true in the time setting as the optimal intervention set might not only be a subset of $\partsXt$ but might also evolve overtime.

\textbf{Bandits and \rl} In the broader decision-making literature, causal relationships have been previously considered in the context of multi-armed bandit problems \citep[\mab,][]{bareinboim2015bandits, lattimore2016causal, lee2018structural, lee2019structural} and reinforcement learning \citep[\rl,][]{lu2018deconfounding, buesing2018woulda, foerster2018counterfactual, zhang2019near, madumal2020explainable}. In these cases, the actions or arms correspond to interventions on an arbitrary causal graph where there exists complex links between the agent’s decisions and the received rewards. Causal \mab algorithms focus on static settings where the distribution of the rewards is stationary and is not affected by the pulled arms. In addition, \mab focus on intervention on discrete variable and only deal with the problem of selecting the right intervention set but not the intervention value. Differently from \our, \rl algorithms explicitly model the state dynamic and account for the way each action affect the state of the environment.
% While \mab algorithms disregard the existence of dynamic variables and do not account for the way in which actions affect the system, causal \rl algorithms explicitly model the state dynamic and account for the actions influence on the environment. 
\our setting differs from both causal \rl and causal \mab. \our does not have a notion of \emph{state} and therefore does not require an explicit model of its dynamic. The system is fully specified by the causal graph and the connected structural equation model. As in \bo, \our does not aim at learning an optimal policy but rather a set of optimal actions.
Furthermore, within each time step, \our allows the agent to perform a number of explorative interventions which are not modifying the environment. Once the optimal action is identified this is propagated in the system thus changing it. Differently from both \mab and \rl, \our is \emph{myopic} that is interventions are decided by maximizing the one-step ahead utility function. We leave the integration of \our with a non-myopic \bo scheme to future work.

\end{document}